\documentclass[11pt]{article}

\usepackage{epsfig,epsf,fancybox}
\usepackage{amsmath}
\usepackage{mathrsfs}
\usepackage{amssymb}
\usepackage{graphicx}
\usepackage{color}
\usepackage{multirow}
\usepackage{paralist}
\usepackage{verbatim}
\usepackage{galois}
\usepackage{algorithm}
\usepackage[noend]{algorithmic}
\usepackage{boxedminipage}
\usepackage{booktabs}
\usepackage{accents}
\usepackage{stmaryrd}
\usepackage[table]{xcolor}
\usepackage{hhline}

\usepackage{subfig}

\usepackage{natbib}

\usepackage{url}
\usepackage[colorlinks,linkcolor=magenta,citecolor=blue, pagebackref=true,backref=true]{hyperref}
\renewcommand*{\backrefalt}[4]{%
    \ifcase #1 \footnotesize{(Not cited.)}%
    \or        \footnotesize{(Cited on page~#2.)}%
    \else      \footnotesize{(Cited on pages~#2.)}%
    \fi}

\newtheorem{theorem}{Theorem}[section]

\newtheorem{lemma}[theorem]{Lemma}
\newtheorem{proposition}[theorem]{Proposition}

\newtheorem{definition}[theorem]{Definition}

\newtheorem{remark}[theorem]{Remark}
\newtheorem{assumption}[theorem]{Assumption}
\numberwithin{equation}{section}

\newcommand{\E}{\mathbb{E}}

\newcommand{\EE}{\mathbb{E}}

\newcommand{\grad}{\nabla}
\newcommand{\gradx}{\nabla_{\mathbf x}}
\newcommand{\grady}{\nabla_{\mathbf y}}
\newcommand{\subg}{\partial}

\newcommand{\proj}{\mathcal{P}}

\newcommand{\x}{\mathbf x}
\newcommand{\y}{\mathbf y}

\newcommand{\z}{\mathbf z}
\newcommand{\w}{\mathbf w}

\newcommand{\argmin}{\mathop{\rm argmin}}
\newcommand{\argmax}{\mathop{\rm argmax}}

\newcommand{\OCal}{\mathcal{O}}

\newcommand{\YCal}{\mathcal{Y}}

\newcommand{\prox}{\textnormal{prox}}

\newcommand{\br}{\mathbb{R}}

\newcommand{\ba}{\begin{array}}
\newcommand{\ea}{\end{array}}

\title{\bf{\LARGE{On Gradient Descent Ascent for Nonconvex-Concave Minimax Problems}}}

\author{Tianyi Lin\thanks{Department of IEOR, UC Berkeley, Berkeley, CA 94720, USA; Email: darren{\_}lin@berkeley.edu.} \and Chi Jin\thanks{Department of EE, Princeton University, Princeton, NJ 08544, USA; Email: chij@princeton.edu.} \and Michael. I. Jordan\thanks{Department of EECS and Statistics, UC Berkeley, Berkeley, CA 94720, USA; Email: jordan@cs.berkeley.edu.}
}

\begin{document}
\maketitle

\begin{abstract}
We consider nonconvex-concave minimax problems, $\min_\x \max_{\y \in \YCal} f(\x, \y)$ where $f$ is nonconvex in $\x$ but concave in $\y$ and $\YCal$ is a convex and bounded set. One of the most popular algorithms for solving this problem is the celebrated gradient descent ascent (GDA) algorithm, which has been widely used in machine learning, control theory and economics. Despite the extensive convergence results for the convex-concave setting, GDA with equal stepsize can converge to limit cycles or even diverge in a general setting. In this paper, we present the complexity results on two-timescale GDA for solving nonconvex-concave minimax problems, showing that the algorithm can find a stationary point of the function $\Phi(\cdot) := \max_{\y\in\YCal} f(\cdot, \y)$ efficiently. To our knowledge, this is the first nonasymptotic analysis for two-timescale GDA in this setting, shedding light on its superior practical performance in training generative adversarial networks (GANs) and other real applications. 
\end{abstract}

\section{Introduction}
We consider the following smooth minimax optimization problem:
\begin{equation}\label{prob:main}
\min_{\x \in \br^m} \max_{\y \in \YCal} \ f(\x, \y), 
\end{equation}
where $f: \br^m \times \br^n \rightarrow \br$ is nonconvex in $\x$ but concave in $\y$ and where $\YCal$ is a convex set. Since von Neumann's seminal work~\citep{Neumann-1928-Theorie}, the problem of finding the solution to problem~\eqref{prob:main} has been a major focus of research in mathematics, economics and computer science~\citep{Basar-1999-Dynamic,Nisan-2007-Algorithmic,Von-2007-Theory}. In recent years, minimax optimization theory has begun to see applications in machine learning, with examples including generative adversarial networks (GANs)~\citep{Goodfellow-2014-Generative}, statistics~\citep{Xu-2009-Robustness, Abadeh-2015-Distributionally}, online learning~\citep{Cesa-2006-Prediction}, deep learning~\citep{Sinha-2018-Certifiable} and distributed computing~\citep{Shamma-2008-Cooperative, Mateos-2010-Distributed}. Moreover, there is increasing awareness that machine-learning systems are embedded in real-world settings involving scarcity or competition that impose game-theoretic constraints~\citep{Jordan-2018-Artificial}. 

One of the simplest candidates for solving problem~\eqref{prob:main} is the natural generalization of gradient descent (GD) known as \emph{gradient descent ascent} (GDA). At each iteration, this algorithm performs gradient descent over the variable $\x$ with the stepsize $\eta_\x$ and gradient ascent over the variable $\y$ with the stepsize $\eta_\y$. On the positive side, when the objective function $f$ is convex in $\x$ and concave in $\y$, there is a vast literature establishing asymptotic and nonasymptotic convergence for the average iterates generated by GDA with the equal stepsizes ($\eta_\x = \eta_\y$);~\citep[see, e.g.,][]{Korpelevich-1976-Extragradient, Chen-1997-Convergence, Nedic-2009-Subgradient, Nemirovski-2004-Prox, Du-2018-Linear}.  Local linear convergence can also be shown under the additional assumption that $f$ is locally strongly convex in $\x$ and strongly concave in $\y$~\citep{Cherukuri-2017-Saddle, Adolphs-2018-Local, Liang-2018-Interaction}. However, there has been no shortage of research highlighting the fact that in a general setting GDA with equal stepsizes can converge to limit cycles or even diverge~\citep{Benaim-1999-Mixed, Hommes-2012-Multiple, Mertikopoulos-2018-Cycles}.  

Recent research has focused on alternative gradient-based algorithms that have guarantees beyond the convex-concave setting~\citep{Daskalakis-2017-Training, Heusel-2017-Gans, Mertikopoulos-2019-Optimistic, Mazumdar-2019-Finding}. Two-timescale GDA~\citep{Heusel-2017-Gans} has been particularly  popular. This algorithm, which involves unequal stepsizes ($\eta_\x \neq \eta_\y$),  has been shown to empirically to alleviate the issues of limit circles and it has theoretical support in terms of local asymptotic convergence to Nash equilibria~\citep[Theorem~2]{Heusel-2017-Gans}. 
\renewcommand{\arraystretch}{1.4}
\begin{table*}[!t]
\centering
\caption{The gradient complexity of all algorithms for nonconvex-(strongly)-concave minimax problems. $\epsilon$ is a tolerance and $\kappa>0$ is a condition number. The result denoted by $^\star$ refers to the complexity bound after translating from $\epsilon$-stationary point of $f$ to our optimality measure; see Propositions~\ref{prop:criterion-nsc} and~\ref{prop:criterion-nc}. The result denoted by $^\circ$ is not presented explicitly but easily derived by standard arguments.}\hspace*{-2em}
\begin{tabular}{|c|c|c|c|c|c|} \hline
& \multicolumn{2}{c|}{Nonconvex-Strongly-Concave} & \multicolumn{2}{c|}{Nonconvex-Concave} & \multirow{2}{*}{Simplicity} \\ \cline{2-5}
& Deterministic & Stochastic & Deterministic & Stochastic & \\ \hhline{|======|}
~\citet{Jin-2019-Minmax} & $\tilde{O}\left(\kappa^2\epsilon^{-2}\right)^\circ$ & $\tilde{O}\left(\kappa^3\epsilon^{-4}\right)$ & $O(\epsilon^{-6})$ & $O(\epsilon^{-8})^\circ$ & Double-loop \\ \hline
~\citet{Rafique-2018-Non} & $\tilde{O}(\kappa^2\epsilon^{-2})$ & $\tilde{O}(\kappa^3\epsilon^{-4})$ & $\tilde{O}(\epsilon^{-6})$ & $\tilde{O}(\epsilon^{-6})$ & Double-loop \\ \hline
~\citet{Nouiehed-2019-Solving} & $\tilde{O}(\kappa^4\epsilon^{-2})^{\star, \circ}$ & -- & $O(\epsilon^{-7})^\star$ & -- & Double-loop \\ \hline
~\citet{Thekumparampil-2019-Efficient} & -- & -- & $\tilde{O}(\epsilon^{-3})$ & -- & Triple-loop \\ \hline
~\citet{Kong-2019-Accelerated} & -- & -- & $\tilde{O}(\epsilon^{-3})$ & -- & Triple-loop \\ 
\hhline{|======|}
~\citet{Lu-2019-Hybrid} & $O(\kappa^4\epsilon^{-2})^\star$ & -- & $O(\epsilon^{-8})^\star$ & -- & Single-loop \\ \hline
\cellcolor{lightgray} \textbf{This paper} & $O(\kappa^2\epsilon^{-2})$ & $O(\kappa^3\epsilon^{-4})$ & $O(\epsilon^{-6})$ & $O(\epsilon^{-8})$ & Single-loop \\ \hline
\end{tabular}\label{tab:results} 
\end{table*}

This asymptotic result stops short of providing an understanding of algorithmic efficiency, and it would be desirable to provide a stronger, nonasymptotic, theoretical convergence rate for two-timescale GDA in a general setting. In particular, the following general structure arises in many applications: $f(\x, \cdot)$ is concave for any $\x$ and $\YCal$ is a bounded set. Two typical examples include training of a neural network which is robust to adversarial examples~\citep{Madry-2017-Towards} and learning of a robust classifier from multiple distributions~\citep{Sinha-2018-Certifiable}. Both of these schemes can be posed as nonconvex-concave minimax problems. Based on this observation, it is natural to ask the question: \textit{Are two-timescale GDA and stochastic GDA (SGDA) provably efficient for nonconvex-concave minimax problems?}

\paragraph{Our results:} This paper presents an affirmative answer to this question, providing nonasymptotic complexity results for two-time scale GDA and SGDA in two settings. In the nonconvex-strongly-concave setting, two-time scale GDA and SGDA require $O(\kappa^2\epsilon^{-2})$ gradient evaluations and $O(\kappa^3\epsilon^{-4})$ stochastic gradient evaluations, respectively, to return an $\epsilon$-stationary point of the function $\Phi(\cdot) = \max_{\y \in \YCal} f(\cdot, \y)$ where $\kappa > 0$ is a condition number. In the nonconvex-concave setting, two-time scale GDA and SGDA require $O(\epsilon^{-6})$ gradient evaluations and $O(\epsilon^{-8})$ stochastic gradient evaluations. 

\paragraph{Main techniques:} To motivate the proof ideas for analyzing two-time scale GDA and SGDA, it is useful to contrast our work with some of the strongest existing convergence analyses for nonconvex-concave problems.  In particular,~\citet{Jin-2019-Minmax} and \citet{Nouiehed-2019-Solving} have provided complexity results for algorithms that have a nested-loop structure.  Specifically, GDmax and multistep GDA are algorithms in which the outer loop can be interpreted as an inexact gradient descent on a nonconvex function $\Phi(\cdot) = \max_{\y \in \YCal} f(\cdot, \y)$ while the inner loop provides an approximate solution to the maximization problem $\max_{\y \in \YCal} f(\x, \y)$ for a given $\x \in \br^m$.  Strong convergence results are obtained when accelerated gradient ascent is used in the maximization problem.

Compared to GDmax and multistep GDA, two-time scale GDA and SGDA are harder to analyze. Indeed, $\y_t$ is not necessarily guaranteed to be close to $\y^\star(\x_t)$ at each iteration and thus it is unclear that $\gradx f(\x_t, \y_t)$ might a reasonable descent direction. To overcome this difficulty, we develop a new technique which analyzes the concave optimization with a slowly changing objective function. This is the main technical contribution of this paper.

\paragraph{Notation.} We use bold lower-case letters to denote vectors and caligraphic upper-case letter to denote sets. We use $\left\|\cdot\right\|$ to denote $\ell_2$-norm of vectors and spectral norm of matrices. For a function $f: \br^n \rightarrow \br$, $\partial f(\z)$ denotes the subdifferential of $f$ at $\z$. If $f$ is differentiable, $\partial f(\z) = \left\{\grad f(\z)\right\}$ where $\grad f(\cdot)$ denotes the gradient of $f$ and $\gradx f(\cdot)$ denotes the partial gradient of $f$ with respect to $\x$. For a symmetric matrix $A \in \br^{n \times n}$, the largest and smallest eigenvalue of $A$ denoted by $\lambda_{\max}(A)$ and $\lambda_{\min}(A)$.

\section{Related Work}
\paragraph{Convex-concave setting.} Historically, an early concrete instantiation of problem~\eqref{prob:main} involved computing a pair of probability vectors $\left(\x, \y\right)$, or equivalently solving $\min_{\x \in \Delta^m} \max_{\y \in \Delta^n} \x^\top A \y$ for a matrix $A \in \br^{m \times n}$ and probability simplices $\Delta^m$ and $\Delta^n$. This bilinear minimax problem together with von Neumann's minimax theorem~\citep{Neumann-1928-Theorie} was a cornerstone in the development of game theory. A simple and generic algorithm scheme was developed for solving this problem in which the min and max players each implemented a simple learning procedure in tandem~\citep{Robinson-1951-Iterative}. After then,~\citet{Sion-1958-General} generalized von Neumann's result from bilinear games to general convex-concave games, $\min_\x \max_\y f(\x, \y) = \max_\y \min_\x f(\x, \y)$, and triggered a line of algorithmic research on convex-concave minimax optimization in both continuous time~\citep{Kose-1956-Solutions, Cherukuri-2017-Saddle} and discrete time~\citep{Uzawa-1958-Iterative, Golshtein-1974-Generalized, Korpelevich-1976-Extragradient, Nemirovski-2004-Prox, Nedic-2009-Subgradient, Mokhtari-2019-Unified, Mokhtari-2019-Proximal, Azizian-2019-Tight}. It is well known that GDA can find an $\epsilon$-approximate saddle point within $O(\kappa^2 \log (1/\epsilon))$ iterations for strongly-convex-strongly-concave games, and $O(\epsilon^{-2})$ iterations for convex-concave games if we impose the diminishing stepsizes~\citep{Nedic-2009-Subgradient, Nemirovski-2004-Prox}.

\paragraph{Nonconvex-concave setting.} Nonconvex-concave minimax problems appear to be a class of tractable problems in the form of problem~\eqref{prob:main} and have emerged as a focus in optimization and machine learning~\citep{Namkoong-2016-Stochastic, Sinha-2018-Certifiable, Rafique-2018-Non, Sanjabi-2018-Convergence, Grnarova-2018-An, Lu-2019-Hybrid, Nouiehed-2019-Solving, Thekumparampil-2019-Efficient, Kong-2019-Accelerated}; see Table~\ref{tab:results} for a comprehensive overview. We also wish to highlight the work of~\citet{Grnarova-2018-An}, who proposed a variant of GDA for nonconvex-concave problem and the work of~\citet{Sinha-2018-Certifiable} and~\citet{Sanjabi-2018-Convergence}, who studied a class of inexact nonconvex SGD algorithms that can be categorized as variants of SGDmax for nonconvex-strongly-concave problem. \citet{Jin-2019-Minmax} analyzed the GDmax algorithm for nonconvex-concave problem and provided nonasymptotic convergence results. 

\citet{Rafique-2018-Non} proposed ``proximally guided stochastic mirror descent'' and ``variance reduced gradient'' algorithms (PGSMD/PGSVRG) and proved that these algorithms find an approximate stationary point of $\Phi(\cdot) :=\max_{\y\in\YCal} f(\cdot, \y)$. However, PGSMD/PGSVRG are nested-loop algorithms and convergence results were established only in the special case where $f(\x, \cdot)$ is a linear function~\citep[Assumption 2 D.2]{Rafique-2018-Non}.~\citet{Nouiehed-2019-Solving} developed a multistep GDA (MGDA) algorithm by incorporating accelerated gradient ascent as the subroutine at each iteration. This algorithm provably finds an approximate stationary point of $f(\cdot, \cdot)$ for nonconvex-concave problems with the fast rate of $O(\epsilon^{-3.5})$. Very recently,~\citet{Thekumparampil-2019-Efficient} have proposed a proximal dual implicit accelerated gradient (ProxDIAG) algorithm for nonconvex-concave problems and proved that the algorithm find an approximate stationary point of $\Phi(\cdot)$ with the rate of $O(\epsilon^{-3})$. This complexity result is also achieved by an inexact proximal point algorithm~\citep{Kong-2019-Accelerated}. All of these algorithms are, however, nested-loop algorithms and thus relatively complicated to implement.  One would like to know whether the nested-loop structure is necessary or whether GDA, a single-loop algorithm, can be guaranteed to converge in the nonconvex-(strongly)-concave setting. 

The most closest work is~\citet{Lu-2019-Hybrid} in which a single-loop HiBSA algorithm for nonconvex-(strongly)-concave problems is proposed with theoretical guarantees under a different notion of optimality. However, their analysis requires some restrictive assumptions; e.g., that $f(\cdot, \cdot)$ is lower bounded.  We only require that $\max_{\y \in \YCal} f(\cdot, \y)$ is lower bounded. An example which meets our conditions and not those of~\citet{Lu-2019-Hybrid} is $\min_{\x \in \br}\max_{\y \in [-1, 1]} \x^\top\y$. Our less-restrictive assumptions make the problem more challenging and our technique is accordingly fundamentally difference from theirs. 

\paragraph{Nonconvex-nonconcave setting.} During the past decade, the study of nonconvex-nonconcave minimax problems has become a central topic in machine learning, inspired in part by the advent of generative adversarial networks~\citep{Goodfellow-2014-Generative} and adversarial learning~\citep{Madry-2017-Towards, Namkoong-2016-Stochastic, Sinha-2018-Certifiable}. Most recent work aims at defining a notion of goodness or the development of new procedures for reducing oscillations~\citep{Daskalakis-2018-Limit, Adolphs-2018-Local, Mazumdar-2019-Finding} and speeding up the convergence of gradient dynamics~\citep{Heusel-2017-Gans, Balduzzi-2018-Mechanics, Mertikopoulos-2019-Optimistic, Lin-2018-Solving}. More specifically,~\citet{Daskalakis-2018-Limit} studied minimax optimization (or zero-sum games) and show that the stable limit points of GDA are not necessarily Nash equilibria. \citet{Adolphs-2018-Local} and~\citet{Mazumdar-2019-Finding} proposed Hessian-based algorithms whose stable fixed points are exactly Nash equilibria. On the other hand,~\citet{Balduzzi-2018-Mechanics} developed a new symplectic gradient adjustment (SGA) algorithm for finding stable fixed points in potential games and Hamiltonian games. \citet{Heusel-2017-Gans} proposed two-timescale GDA and show that Nash equilibria are stable fixed points of the continuous limit of two-timescale GDA under certain strong conditions. All of the existing convergence results are either local or asymptotic and can not be extended to cover our results in a nonconvex-concave setting. Very recently,~\citet{Mertikopoulos-2019-Optimistic} and~\citet{Lin-2018-Solving} provide nonasymptotic guarantees for a special class of nonconvex-nonconcave minimax problems under variational stability and the Minty condition. However, while both of these two conditions must hold in convex-concave setting, they do not necessarily hold in nonconvex-(strongly)-concave problem. 

\paragraph{Online learning setting.} From the online learning perspective, it is crucial to understand if the proposed algorithm achieves no-regret property. For example, the optimistic algorithm~\citep{Daskalakis-2018-Last} is a no-regret algorithm, while the extragradient algorithm~\citep{Mertikopoulos-2019-Optimistic} is not. In comparing limit behavior of zero-sum game dynamics,~\citet{Bailey-2018-Multiplicative} showed that the multiplicative weights update has similar property as GDA and specified the necessity of introducing the optimistic algorithms to study the last-iterate convergence.


\section{Preliminaries}\label{sec:prelim}
We recall basic definitions for smooth functions.
\begin{definition}
A function $f$ is $L$-Lipschitz if for $\forall \x, \x'$, we have $\left\| f(\x) -  f (\x')\right\| \leq L\left\|\x-\x'\right\|$.
\end{definition} 
\begin{definition}
A function $f$ is $\ell$-smooth if for $\forall \x, \x'$, we have $\left\|\grad f(\x) - \grad f (\x')\right\| \leq \ell\left\|\x-\x'\right\|$.
\end{definition}

Recall that the minimax problem~\eqref{prob:main} is equivalent to minimizing a function $\Phi(\cdot) = \max_{\y \in \YCal} f(\cdot, \y)$. For nonconvex-concave minimax problems in which $f(\x, \cdot)$ is concave for each $\x \in \br^m$, the maximization problem $\max_{\y \in \YCal} f(\x, \y)$ can be solved efficiently and provides useful information about $\Phi$. However, it is still NP hard to find the global minimum of $\Phi$ in general since $\Phi$ is nonconvex. 

\paragraph{Objectives in this paper.} We start by defining local surrogate for the global minimum of $\Phi$. A common surrogate in nonconvex optimization is the notion of stationarity, which is appropriate if $\Phi$ is differentiable.
\begin{definition}\label{def:nsc-stationary}
A point $\x$ is an $\epsilon$-\emph{stationary point} ($\epsilon \geq 0$) of a differentiable function $\Phi$ if $\|\grad\Phi(\x)\| \leq \epsilon$. If $\epsilon = 0$, then $\x$ is a stationary point. 
\end{definition}
Definition~\ref{def:nsc-stationary} is sufficient for nonconvex-strongly-concave minimax problem since $\Phi(\cdot) = \max_{\y \in \YCal} f(\cdot, \y)$ is differentiable in that setting. In contrast, a function $\Phi$ is not necessarily differentiable for general nonconvex-concave minimax problem even if $f$ is Lipschitz and smooth. A weaker condition that we make use of is the following.
\begin{definition}\label{def:nc-weak-convex}
A function $\Phi$ is $\ell$-\emph{weakly convex} if a function $\Phi(\cdot) + (\ell/2)\|\cdot\|^2$ is convex.
\end{definition}
For a $\ell$-weakly convex function $\Phi$, the subdifferential $\partial\Phi$ is uniquely determined by the subdifferential of $\Phi + (\ell/2)\|\cdot\|^2$. Thus, a naive measure of approximate stationarity can be defined as a point $\x \in \br^m$ such that at least one subgradient is small: $\min_{\xi \in \partial \Phi(\x)} \|\xi\| \leq \epsilon$. However, this notion of stationarity can be very restrictive when optimizing nonsmooth functions. For example, when $\Phi(\cdot) = |\cdot|$ is a one-dimensional function, an $\epsilon$-stationary point is zero for all $\epsilon \in [0, 1)$. This means that finding an approximate stationary point under this notion is as difficult as solving the problem exactly. In respond to this issue,~\citet{Davis-2019-Stochastic} propose an alternative notion of stationarity based on the Moreau envelope.  This has become recognized as standard for optimizing a weakly convex function.
\begin{definition}
A function $\Phi_\lambda: \br^m \rightarrow \br$ is the Moreau envelope of $\Phi$ with a positive parameter $\lambda>0$ if $\Phi_\lambda(\x) = \min_\w \Phi(\w) + (1/2\lambda)\|\w - \x\|^2$ for each $\x \in \br^m$. 
\end{definition}
\begin{lemma}\label{Lemma:nc-moreau-envelope}
If $f$ is $\ell$-smooth and $\YCal$ is bounded, the Moreau envelope $\Phi_{1/2\ell}$ of $\Phi(\cdot) = \max_{\y \in \YCal} f(\cdot, \y)$ is differentiable with $\grad\Phi_{1/2\ell}(\cdot) = 2\ell(\cdot - \prox_{\Phi/2\ell}(\cdot))$.
\end{lemma}
An alternative measure of approximate stationarity of $\Phi(\cdot) = \max_{\y \in \YCal} f(\cdot, \y)$ can be defined as a point $\x$ such that the norm of the gradient of Moreau envelope is small: $\|\grad\Phi_{1/2\ell}\| \leq \epsilon$. That being said, 
\begin{definition}\label{def:nc-stationary}
A point $\x$ is an $\epsilon$-stationary point of a $\ell$-weakly convex function $\Phi$ if $\|\grad\Phi_{1/2\ell}(\x)\| \leq \epsilon$. If $\epsilon = 0$, then $\x$ is a stationary point.  
\end{definition}
Although Definition~\ref{def:nc-stationary} is based on the Moreau envelopes, it also connects to the function $\Phi$ as follows. 
\begin{lemma}\label{Lemma:nc-stationary}
If $\x$ is an $\epsilon$-stationary point of a $\ell$-weakly convex function $\Phi$ (Definition \ref{def:nc-stationary}), there exists $\hat{\x} \in \br^m$ such that $\min_{\xi \in \partial \Phi(\hat{\x})} \|\xi\| \leq \epsilon$ and $\|\x - \hat{\x}\| \leq \epsilon/2\ell$.  
\end{lemma}
Lemma~\ref{Lemma:nc-stationary} shows that an $\epsilon$-stationary point defined by Definition~\ref{def:nc-stationary} can be interpreted as the relaxation or surrogate for $\min_{\xi \in \partial \Phi(\x)} \|\xi\| \leq \epsilon$. In particular, if a point $\x$ is an $\epsilon$-stationary point of an $\ell$-weakly convex function $\Phi$, then $\x$ is close to a point $\hat{\x}$ which has at least one small subgradient. 
\begin{remark}
We remark that our notion of stationarity is natural in real scenarios. Indeed, many applications arising from adversarial learning can be formulated as the minimax problem~\eqref{prob:main}, and, in this setting, $\x$ is the classifier while $\y$ is the adversarial noise for the data. Practitioners are often interested in finding a robust classifier $\x$ instead of recovering the adversarial noise $\y$. Any stationary point of the function $\Phi(\cdot) = \max_{\y \in \YCal} f(\cdot, \y)$ corresponds precisely to a robust classifier that achieves better classification error. 
\end{remark}
\begin{remark}
There are also other notions of stationarity based on $\grad f$ are proposed for nonconvex-concave minimax problems in the literature~\citep{Lu-2019-Hybrid, Nouiehed-2019-Solving}. However, as pointed by~\citet{Thekumparampil-2019-Efficient},  these notions are weaker than that defined in Definition~\ref{def:nsc-stationary} and~\ref{def:nc-stationary}. For the sake of completeness, we specify the relationship between our notion of stationarity and other notions in Proposition~\ref{prop:criterion-nsc} and~\ref{prop:criterion-nc}. 
\end{remark}
\begin{algorithm}[!t]
\caption{Two-Timescale GDA}\label{Algorithm:GDA}
\begin{algorithmic}
\STATE \textbf{Input:} $(\x_0, \y_0)$, stepsizes $(\eta_\x, \eta_\y)$. 
\FOR{$t = 1, 2, \ldots, T$}
\STATE $\x_t \leftarrow \x_{t-1} - \eta_\x \gradx f(\x_{t-1}, \y_{t-1})$,
\STATE $\y_t \leftarrow \proj_\YCal\left(\y_{t-1} + \eta_\y \grady f(\x_{t-1}, \y_{t-1})\right)$. 
\ENDFOR
\STATE Randomly draw $\hat{\x}$ from $\{\x_t\}_{t=1}^T$ at uniform. 
\STATE \textbf{Return:} $\hat{\x}$. 
\end{algorithmic}
\end{algorithm}
\begin{algorithm}[!t]
\caption{Two-Timescale SGDA}\label{Algorithm:SGDA}
\begin{algorithmic}
\STATE \textbf{Input:} $(\x_0, \y_0)$, stepsizes $(\eta_\x, \eta_\y)$, batch size $M$.  
\FOR{$t = 1, 2, \ldots, T$}
\STATE Draw a collection of i.i.d. data samples $\{\xi_i\}_{i=1}^M$. 
\STATE $\x_t \leftarrow \x_{t-1} - \eta_\x\left(\frac{1}{M}\sum_{i=1}^{M} G_\x(\x_{t-1}, \y_{t-1}, \xi_i)\right)$. 
\STATE $\y_t \leftarrow \proj_\YCal\left(\y_{t-1} + \eta_\y(\frac{1}{M} \sum_{i=1}^{M} G_\y(\x_{t-1}, \y_{t-1}, \xi_i))\right)$. 
\ENDFOR
\STATE Randomly draw $\hat{\x}$ from $\{\x_t\}_{t=1}^T$ at uniform. 
\STATE \textbf{Return:} $\hat{\x}$. 
\end{algorithmic}
\end{algorithm}

\section{Main Results}\label{sec:results}
In this section, we present complexity results for two-timescale GDA and SGDA in the setting of nonconvex-strongly-concave and nonconvex-concave minimax problems.

The algorithmic schemes that we study are extremely simple and are presented in Algorithm~\ref{Algorithm:GDA} and~\ref{Algorithm:SGDA}. In particular, each iteration comprises one (stochastic) gradient descent step over $\x$ with the stepsize $\eta_\x > 0$ and one (stochastic) gradient ascent step over $\y$ with the stepsize $\eta_\y > 0$. The choice of stepsizes $\eta_\x$ and $\eta_\y$ is crucial for the algorithms in both theoretical and practical senses. In particular, classical GDA and SGDA assume that $\eta_\x = \eta_\y$, and the last iterate is only known convergent in strongly convex-concave problems~\citep{Liang-2018-Interaction}. Even in convex-concave settings (or bilinear settings as special cases), GDA requires the assistance of averaging or other strategy~\citep{Daskalakis-2018-Last} to converge, otherwise, with fixed stepsize, the last iterate will always diverge and hit the constraint boundary eventually~\citep{Daskalakis-2017-Training, Mertikopoulos-2018-Cycles, Daskalakis-2018-Last}. In contrast, two-timescale GDA and SGDA ($\eta_\x \neq \eta_\y$) were shown to be locally convergent and practical in training GANs~\citep{Heusel-2017-Gans}. 

One possible reason for this phenomenon is that the choice of $\eta_\x \neq \eta_\y$ reflects the nonsymmetric nature of nonconvex-(strongly)-concave problems. For sequential problems such as robust learning, where the natural order of min-max is important (i.e., min-max is not equal to max-min), practitioners often prefer faster convergence for the inner max problem. Therefore, it is reasonable for us to choose $\eta_\x \ll \eta_\y$ rather than $\eta_\x = \eta_\y$. 

Finally, we make the assumption that the oracle $G = (G_\x, G_\y)$ is unbiased and has bounded variance. 
\begin{assumption}\label{Assumption:stoc-oracle}
The stochastic oracle $G$ satisfies
\begin{equation*}
\EE[G(\x, \y, \xi) - \grad f(\x, \y] = 0, \qquad \EE[\|G(\x, \y, \xi) - \grad f(\x, \y)\|^2] \leq \sigma^2. 
\end{equation*}
\end{assumption}

\subsection{Nonconvex-strongly-concave minimax problems}\label{sec:results_sc}
We present the complexity results for two-time-scale GDA and SGDA in the setting of nonconvex-strongly-concave minimax problems. The following assumption is made throughout this subsection. 
\begin{assumption}\label{Assumption:nsc} 
The objective function and constraint set $\left(f: \br^{m+n} \rightarrow \br, \ \YCal \subseteq \br^n\right)$ satisfy
\begin{enumerate}
\item $f$ is $\ell$-smooth and $f(\x, \cdot)$ is $\mu$-strongly concave. 
\item $\YCal$ is a convex and bounded set with a diameter $D \geq 0$.
\end{enumerate}
\end{assumption}
Let $\kappa =\ell/\mu$ denote the condition number and define 
\begin{equation*}
\Phi(\cdot) = \max_{\y \in \YCal} f(\cdot, \y), \quad \y^\star(\cdot) = \argmax_{\y \in \YCal} f(\cdot, \y). 
\end{equation*}
We present a lemma on the structure of the function $\Phi$ in the nonconvex-strongly-concave setting. 
\begin{lemma}\label{Lemma:nsc-structure} 
Under Assumption~\ref{Assumption:nsc}, $\Phi(\cdot)$ is $(\ell + \kappa\ell)$-smooth with $\grad\Phi(\cdot) = \gradx f(\cdot, \y^\star(\cdot))$. Also, $\y^\star(\cdot)$ is $\kappa$-Lipschitz. 
\end{lemma}
Since $\Phi$ is differentiable, the notion of stationarity in Definition~\ref{def:nsc-stationary} is our target given only access to the (stochastic) gradient of $f$. Denote $\Delta_\Phi = \Phi(\x_0) - \min_\x \Phi(\x)$, we proceed to provide theoretical guarantees for two-timescale GDA and SGDA algorithms. 
\begin{theorem}[GDA]\label{Theorem:nsc-GDA-complexity-bound}
Under Assumption~\ref{Assumption:nsc} and letting the stepsizes be chosen as $\eta_\x = \Theta(1/\kappa^2\ell)$ and $\eta_\y = \Theta(1/\ell)$, the iteration complexity (also the gradient complexity) of Algorithm \ref{Algorithm:GDA} to return an $\epsilon$-stationary point is bounded by 
\begin{equation*}
O\left(\frac{\kappa^2\ell\Delta_\Phi + \kappa\ell^2 D^2}{\epsilon^2}\right).
\end{equation*}
\end{theorem} 
\begin{theorem}[SGDA]\label{Theorem:nsc-SGDA-complexity-bound}
Under Assumption~\ref{Assumption:stoc-oracle} and~\ref{Assumption:nsc} and letting the stepsizes $\eta_\x, \eta_\y$ be chosen as the same in Theorem~\ref{Theorem:nsc-GDA-complexity-bound} with the batch size $M = \Theta(\max\{1, \kappa\sigma^2\epsilon^{-2}\})$, the iteration complexity of Algorithm \ref{Algorithm:SGDA} to return an $\epsilon$-stationary point is bounded by 
\begin{equation*}
O\left(\frac{\kappa^2\ell\Delta_\Phi + \kappa\ell^2 D^2}{\epsilon^2}\right), 
\end{equation*}
which gives the total stochastic gradient complexity: 
\begin{equation*}
O\left(\frac{\kappa^2\ell\Delta_\Phi + \kappa\ell^2 D^2}{\epsilon^2}\max\left\{1, \ \frac{\kappa\sigma^2}{\epsilon^2}\right\}\right). 
\end{equation*}
\end{theorem}
\begin{remark}
First, two-timescale GDA and SGDA are guaranteed to find an $\epsilon$-stationary point of $\Phi(\cdot)$ within $O(\kappa^2\epsilon^{-2})$ gradient evaluations and $O(\kappa^3\epsilon^{-4})$ stochastic gradient evaluations, respectively. The ratio of stepsizes $\eta_\y/\eta_\x$ is required to be $\Theta(\kappa^2)$ due to the nonsymmetric nature of our problem (min-max is not equal to max-min). The quantity $O(\kappa^2)$ reflects an efficiency trade-off in the algorithm.

Furthermore, both of the algorithms are only guaranteed to visit an $\epsilon$-stationary point within a certain number of iterations and return $\hat{\x}$ which is drawn from $\{\x_t\}_{t=1}^T$ at uniform. This does not mean that the last iterate $\x_T$ is the $\epsilon$-stationary point. Such a scheme and convergence result are standard in nonconvex optimization for GD or SGD to find stationary points. In practice, one usually returns the iterate when the learning curve stops changing significantly. 

Finally, the minibatch size $M = \Theta(\epsilon^{-2})$ is necessary for the convergence property of two-timescale SGDA. Even though our proof technique can be extended to the purely stochastic setting ($M=1$), the complexity result becomes worse, i.e., $O(\kappa^3\epsilon^{-5})$. It remains open whether this gap can be closed or not and we leave it as  future work. 
\end{remark}

\subsection{Nonconvex-concave minimax problems}\label{sec:results_general}
We present the complexity results for two-timescale GDA and SGDA in the nonconvex-concave minimax setting. The following assumption is made throughout this subsection. 
\begin{assumption}\label{Assumption:nc}
The objective function and constraint set, $\left(f: \br^{m+n} \rightarrow \br, \ \YCal \subset \br^n\right)$ satisfy
\begin{enumerate}
\item $f$ is $\ell$-smooth and $f(\cdot, \y)$ is $L$-Lipschitz for each $\y \in \YCal$ and $f(\x, \cdot)$ is concave for each $\x \in \br^m$.
\item $\YCal$ is a convex and bounded set with a diameter $D \geq 0$. 
\end{enumerate}
\end{assumption}  
Since $f(\x, \cdot)$ is concave for each $\x \in \br^m$, the function $\Phi(\cdot) = \max_{\y \in \YCal} f(\cdot, \y)$ is possibly not differentiable. Fortunately, the following structural lemma shows that $\Phi$ is $\ell$-weakly convex and $L$-Lipschitz. 
\begin{lemma}\label{Lemma:nc-structure}
Under Assumption~\ref{Assumption:nc}, $\Phi(\cdot)$ is $\ell$-weakly convex and $L$-Lipschitz with $\gradx f(\cdot, \y^\star(\cdot)) \in \partial\Phi(\cdot)$ where $\y^\star(\cdot) \in \argmax_{\y \in \YCal} f(\cdot, \y)$. 
\end{lemma}
Since $\Phi$ is $\ell$-weakly convex, the notion of stationarity in Definition~\ref{def:nc-stationary} is our target given only access to the (stochastic) gradient of $f$. Denote $\widehat{\Delta}_\Phi = \Phi_{1/2\ell}(\x_0) - \min_\x \Phi_{1/2\ell}(\x)$ and $\widehat{\Delta}_0 = \Phi(\x_0) - f(\x_0, \y_0)$, we present complexity results for two-timescale GDA and SGDA algorithms. 
\begin{theorem}[GDA]\label{Theorem:nc-GDA-complexity-bound}
Under Assumption~\ref{Assumption:nc} and letting the step sizes be chosen as $\eta_\x = \Theta(\epsilon^4/(\ell^3 L^2 D^2))$ and $\eta_\y = \Theta(1/\ell)$, the iteration complexity (also the gradient complexity) of Algorithm \ref{Algorithm:GDA} to return an $\epsilon$-stationary point is bounded by
\begin{equation*}
O\left(\frac{\ell^3 L^2 D^2\widehat{\Delta}_\Phi}{\epsilon^6} + \frac{\ell^3 D^2\widehat{\Delta}_0}{\epsilon^4}\right).
\end{equation*}
\end{theorem}
\begin{theorem}[SGDA]\label{Theorem:nc-SGDA-complexity-bound}
Under Assumption~\ref{Assumption:stoc-oracle} and~\ref{Assumption:nc} and letting the step sizes be chosen as $\eta_\x =  \Theta(\epsilon^4/(\ell^3D^2(L^2 + \sigma^2)))$ and $\eta_\y = \Theta(\epsilon^2/\ell\sigma^2)$ with the batchsize $M=1$, the iteration complexity (also the stochastic gradient complexity) of Algorithm \ref{Algorithm:SGDA} to return an $\epsilon$-stationary point is bounded by 
\begin{equation*}
\OCal\left(\left(\frac{\ell^3 (L^2 + \sigma^2)D^2\widehat{\Delta}_\Phi}{\epsilon^6} + \frac{\ell^3 D^2\widehat{\Delta}_0}{\epsilon^4}\right)\max\left\{1, \ \frac{\sigma^2}{\epsilon^2}\right\}\right).   
\end{equation*}
\end{theorem}
We make several additional remarks. First of all, two-timescale GDA and SGDA are guaranteed to find an $\epsilon$-stationary point  in terms of Moreau envelopes within $O(\epsilon^{-6})$ gradient evaluations and $O(\epsilon^{-8})$ stochastic gradient evaluations, respectively. The ratio of stepsizes $\eta_\y/\eta_\x$ is required to be $\Theta(1/\epsilon^4)$ and this quantity reflects an efficiency trade-off in the algorithm. Furthermore, similar arguments as in Section \ref{sec:results_sc} hold for the output of the algorithms here. Finally, the minibatch size $M = 1$ is allowed in Theorem~\ref{Theorem:nc-SGDA-complexity-bound}, which is different from the result in Theorem~\ref{Theorem:nsc-SGDA-complexity-bound}. 

\subsection{Relationship between the stationarity notions}
We provide additional technical results on the relationship between our notions of stationarity and other notions based on $\grad f$ in the literature~\citep{Lu-2019-Hybrid, Nouiehed-2019-Solving}. In particular, we show that two notions can be translated in both directions with extra computational cost.  
\begin{definition}\label{def:stationary}
A pair of points $(\x, \y)$ is an $\epsilon$-stationary point ($\epsilon \geq 0$) of a differentiable function $\Phi$ if, for $\y^+ = \proj_\YCal(\y+(1/\ell)\grady f(\x, \y))$, we have
\begin{equation*}
\|\gradx f(\x, \y^+)\| \leq \epsilon, \quad \|\y^+ - \y\| \leq \epsilon/\ell. 
\end{equation*}
\end{definition}
We present our results in the following two propositions. 
\begin{proposition}\label{prop:criterion-nsc} 
Under Assumption \ref{Assumption:nsc}, if a point $\hat{\x}$ is an $\epsilon$-stationary point in terms of Definition~\ref{def:nsc-stationary}, an $O(\epsilon)$-stationary point $(\x', \y')$ in terms of Definition~\ref{def:stationary} can be obtained using additional $O(\kappa\log(1/\epsilon))$ gradients or $O(\epsilon^{-2})$ stochastic gradients. Conversely, if a point $(\hat{\x}, \hat{\y})$ is an $\epsilon/\kappa$-stationary point in terms of Definition~\ref{def:stationary}, a point $\hat{\x}$ is an $O(\epsilon)$-stationary point in terms of Definition~\ref{def:nsc-stationary}. 
\end{proposition}
\begin{proposition}\label{prop:criterion-nc} 
Under Assumption \ref{Assumption:nc}, if a point $\hat{\x}$ is an $\epsilon$-stationary point in terms of Definition~\ref{def:nc-stationary}, an $O(\epsilon)$-stationary point $(\x', \y')$ in terms of Definition~\ref{def:stationary} can be obtained using additional $O(\epsilon^{-2})$ gradients or $O(\epsilon^{-4})$ stochastic gradients. Conversely, if a point $(\hat{\x}, \hat{\y})$ is an $\epsilon^2/\ell D$-stationary point in terms of Definition~\ref{def:stationary}, a point $\hat{\x}$ is an $O(\epsilon)$-stationary point in terms of Definition~\ref{def:nsc-stationary}. 
\end{proposition}
To translate the notion of stationarity based on $\grad f$ to our notion of stationarity, we need to pay an additional factor of $O(\kappa\log(1/\epsilon))$ or $O(\epsilon^{-2})$ in the two settings. In this sense, our notion of stationarity is stronger than the notion based on $\grad f$ in the literature; see~\citet{Lu-2019-Hybrid, Nouiehed-2019-Solving}. We defer the proofs of these propositions to Appendix~\ref{sec:opt}. 

\subsection{Discussions}
Note that the focus of this paper is to provide basic nonasymptotic guarantees for the simple, and widely-used, two-timescale GDA and SGDA algorithms in the nonconvex-(strongly)-concave settings. We do not wish to imply that these algorithms are optimal in any sense, nor that acceleration should necessarily be achieved by incorporating momentum into the update for the variable $\y$. In fact, the optimal rate for optimizing a nonconvex-(strongly)-concave function remains open. The best known complexity bound has been presented by~\citet{Thekumparampil-2019-Efficient} and \citet{Kong-2019-Accelerated}. Both of the analyses only require $\tilde{O}(\epsilon^{-3})$ gradient computations for solving nonconvex-concave problems but suffer from rather complicated algorithmic schemes.  The general question of the construction of optimal algorithms in nonconvex-concave problems is beyond the scope of this paper. 

Second, our complexity results are also valid in the convex-concave setting and this does not contradict results showing the divergence of GDA with fixed stepsize. We note a few distinctions: (1) our results guarantee that GDA will visit $\epsilon$-stationary points at some iterates, which are not necessarily the last iterates; (2) our results only guarantee stationarity in terms of $\x_t$, not $(\x_t, \y_t)$. In fact, our proof permits the possibility of significant changes in $\y_t$ even when $\x_t$ is already close to stationarity. This together with our choice $\eta_\x \ll \eta_\y$, makes our results valid. To this end, we highlight that our algorithms can be used to achieve an approximate Nash equilibrium for convex-concave functions (i.e., optimality for both $\x$ and $\y$). Instead of averaging, we run two passes of two-timescale GDA or SGDA for min-max problem and max-min problem separately. That is, in the first pass we use $\eta_\x \ll \eta_\y$ while in the second pass we use $\eta_\x \gg \eta_\y$. Either pass will return an approximate stationary point for each players, which jointly forms an approximate Nash equilibrium.

\section{Overview of Proofs}\label{Sec:Overview}
In this section, we sketch the complexity analysis for two-timescale GDA (Theorems~\ref{Theorem:nsc-GDA-complexity-bound} and~\ref{Theorem:nc-GDA-complexity-bound}). 

\subsection{Nonconvex-strongly-concave minimax problems}
In the nonconvex-strongly-concave setting, our proof involves setting a pair of stepsizes, $(\eta_\x, \eta_\y)$, which force $\{\x_t\}_{t \geq 1}$ to move much more slowly than $\{\y_t\}_{t \geq 1}$. Recall Lemma~\ref{Lemma:nsc-structure}, which guarantees that $\y^\star(\cdot)$ is $\kappa$-Lipschitz: 
\begin{equation*}
\|\y^\star(\x_1) - \y^\star(\x_2)\| \leq \kappa \|\x_1 - \x_2\|. 
\end{equation*}
If $\{\x_t\}_{t \geq 1}$ moves slowly, then $\{\y^\star(\x_t)\}_{t \geq 1}$ also moves slowly. This allows us to perform gradient ascent on a slowly changing strongly-concave function $f(\x_t, \cdot)$, guaranteeing that $\|\y_t - \y^\star(\x_t)\|$ is small in an amortized sense. More precisely, letting the error be $\delta_t = \|\y^\star(\x_t) - \y_t\|^2$, the standard analysis of inexact nonconvex gradient descent implies a descent inequality in which the sum of $\delta_t$ provides control: 
\begin{equation*}
\Phi(\x_{T+1}) - \Phi(\x_0) \leq -\Omega(\eta_{\x})\left(\sum_{t=0}^T \|\grad \Phi(\x_t)\|^2\right) + O(\eta_\x\ell^2)\left(\sum_{t=0}^T \delta_t\right).
\end{equation*}
The remaining step is to show that the second term is always small compared to the first term on the right-hand side. This can be done via a recursion for $\delta_t$ as follows: 
\begin{equation*}
\delta_t \leq \gamma\delta_{t-1} + \beta\|\grad\Phi(\x_{t-1})\|^2,
\end{equation*}
where $\gamma<1$ and $\beta$ is small. Thus, $\delta_t$ exhibits a linear contraction and $\sum_{t=0}^T \delta_t$ can be controlled by the term $\sum_{t=0}^T \|\grad \Phi(\x_t)\|^2$.  

\subsection{Nonconvex-concave minimax problems}
In this setting, the main idea is again to set a pair of learning rates $(\eta_\x, \eta_\y)$ which force $\{\x_t\}_{t \geq 1}$ to move more slowly than $\{\y_t\}_{t \geq 1}$. However, $f(\x, \cdot)$ is concave and $\y^\star(\cdot)$ is not unique. This means that, even if $\x_1, \x_2$ are extremely close, $\y^\star(\x_1)$ can be dramatically different from $\y^\star(\x_2)$. Thus, $\|\y_t - \y^\star(\x_t)\|$ is no longer a viable error to control.

Fortunately, Lemma~\ref{Lemma:nc-structure} implies that $\Phi$ is Lipschitz. That is to say, when the stepsize $\eta_\x$ is very small, $\{\Phi(\x_t)\}_{t \geq 1}$ moves slowly: 
\begin{equation*}
|\Phi(\x_t) - \Phi(\x_{t-1})| \leq L\|\x_t - \x_{t-1}\| \leq \eta_\x L^2. 
\end{equation*}
Again, this allows us to perform gradient ascent on a slowly changing concave function $f(\x_t, \cdot)$, and guarantees that $\Delta_t = f(\x_t, \z) - f(\x_t, \y_t)$ is small in an amortized sense where $\z \in \y^\star(\x_t)$. The analysis of inexact nonconvex subgradient descent~\citep{Davis-2019-Stochastic} implies that $\Delta_t$ comes into the following descent inequality: 
\begin{equation*}
\Phi_{1/2\ell}(\x_{T+1}) - \Phi_{1/2\ell}(\x_0) \leq O(\eta_\x\ell)\left(\sum_{t=0}^T \Delta_t\right) + O(\eta_\x^2 \ell L^2(T+1)) - O(\eta_{\x})\left(\sum_{t=0}^T \|\grad \Phi_{1/2\ell}(\x_t)\|^2\right) , 
\end{equation*}
where the first term on the right-hand side is the error term. The remaining step is again to show the error term is small compared to the sum of the first two terms on the right-hand side. To bound the term $\sum_{t=0}^T \Delta_t$, we recall the following inequalities and use a telescoping argument (where the optimal point $\y^\star$ does not change):
\begin{equation} \label{eq:convexcontrol}
\Delta_t \leq \tfrac{\|\y_t - \y^\star\|^2 - \|\y_{t+1} - \y^\star\|^2}{\eta_\y}.
\end{equation}
The major challenge here is that the optimal solution $\y^\star(\x_t)$ can change dramatically and the telescoping argument does not go through. An important observation is, however, that~\eqref{eq:convexcontrol} can be proved if we replace the $\y^\star$ by any $\y \in \YCal$, while paying an additional cost that depends on the difference in function value between $\y^\star$ and $\y$. More specifically, we pick a block of size $B = O(\epsilon^2/\eta_\x)$ and show that the following statement holds for any $s \leq \forall t < s + B$,
\begin{equation*}
\Delta_{t-1} \leq O(\ell)\left(\|\y_t - \y^\star(\x_s)\|^2 - \|\y_{t+1} - \y^\star(\x_s)\|^2\right) + O(\eta_\x L^2)(t-1-s). 
\end{equation*}
We perform an analysis on the blocks where the concave problems are similar so the telescoping argument can now work. By carefully choosing $\eta_\x$, the term $\sum_{t=0}^T \Delta_t$ can also be well controlled. 

\begin{figure*}[!t]
\centering
\subfloat[MNIST]{%
\includegraphics[width=0.33\textwidth]{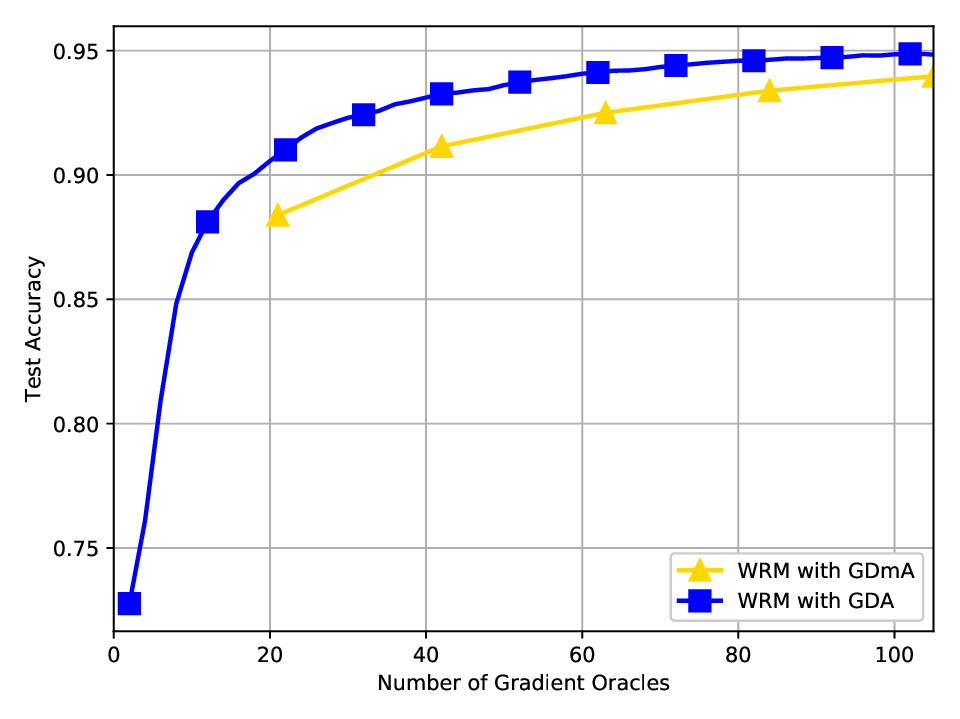}
}
\subfloat[Fashion-MNIST]{%
\includegraphics[width=0.33\textwidth]{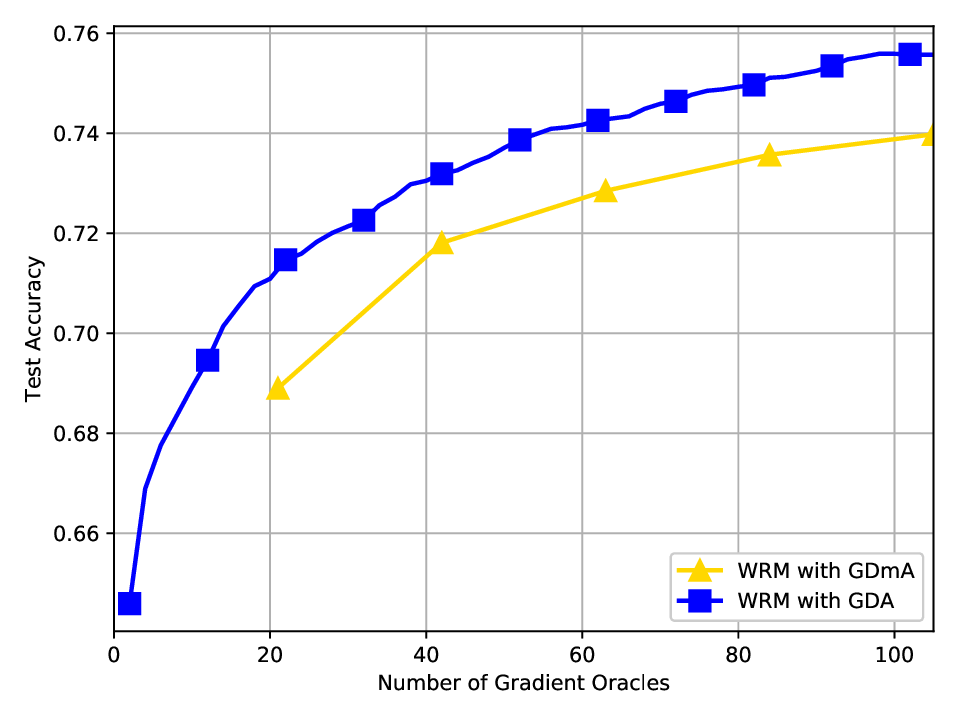}
}
\subfloat[CIFAR-10]{%
\includegraphics[width=0.33\textwidth]{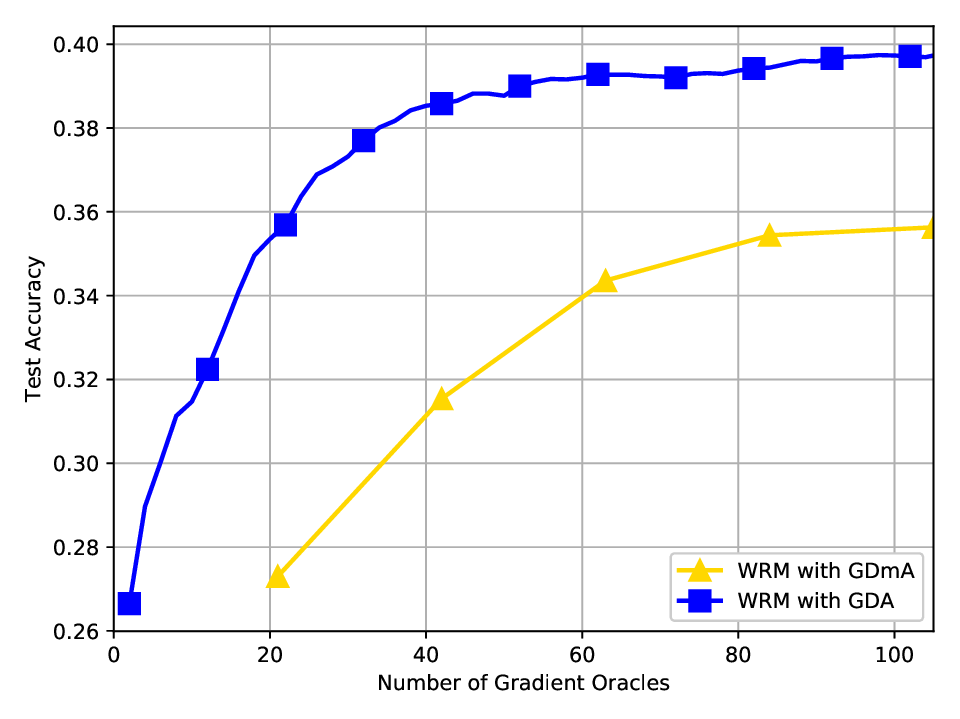}
}
\caption{Performance of WRM with GDmax (i.e., GDmA) and GDA on MNIST, Fashion-MNIST and CIFAR-10 datasets. We demonstrate test classification accuracy vs. time for different WRM models with GDmax and GDA. Note that $\gamma = 0.4$.}\label{Fig:Large-Data}\vspace*{-1em}
\end{figure*}
\begin{figure*}[!t]
\centering
\subfloat[MNIST]{%
\includegraphics[width=0.33\textwidth]{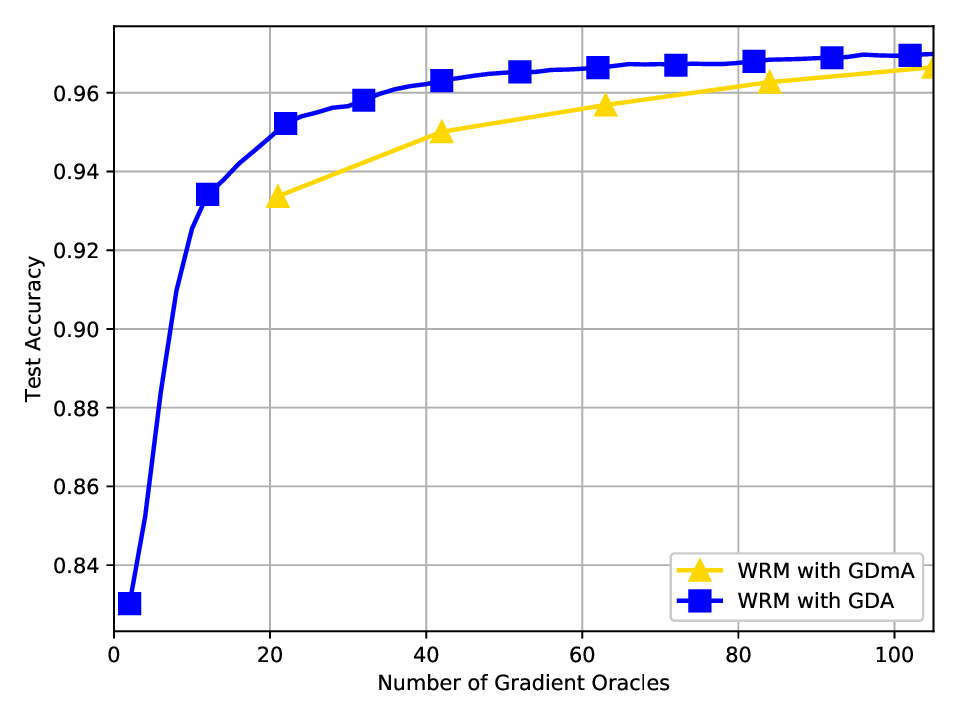}
}
\subfloat[Fashion-MNIST]{%
\includegraphics[width=0.33\textwidth]{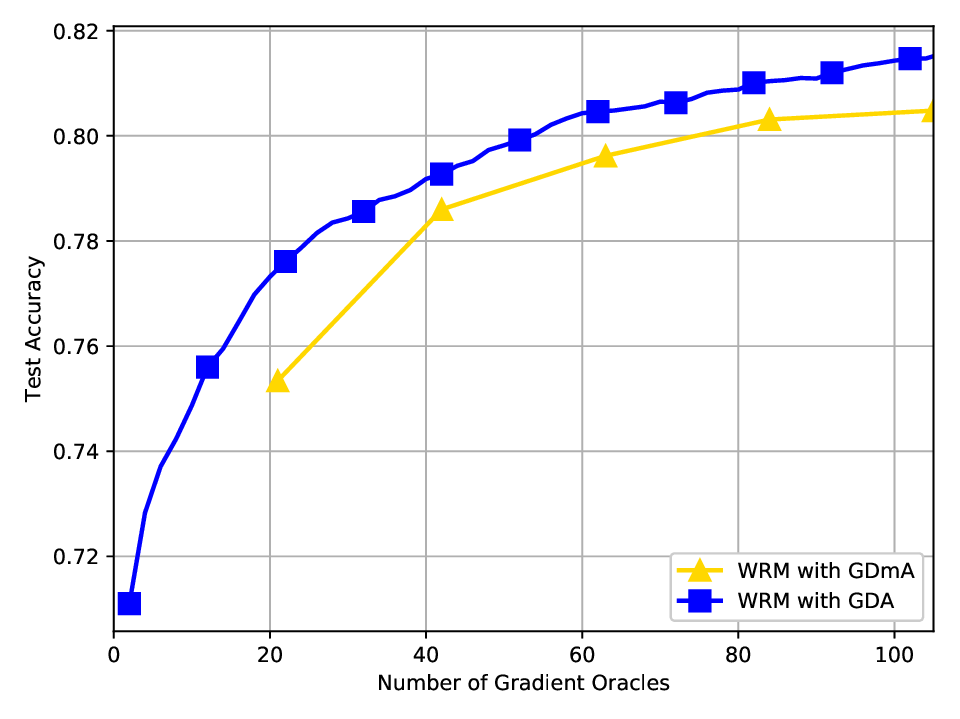}
}
\subfloat[CIFAR-10]{%
\includegraphics[width=0.33\textwidth]{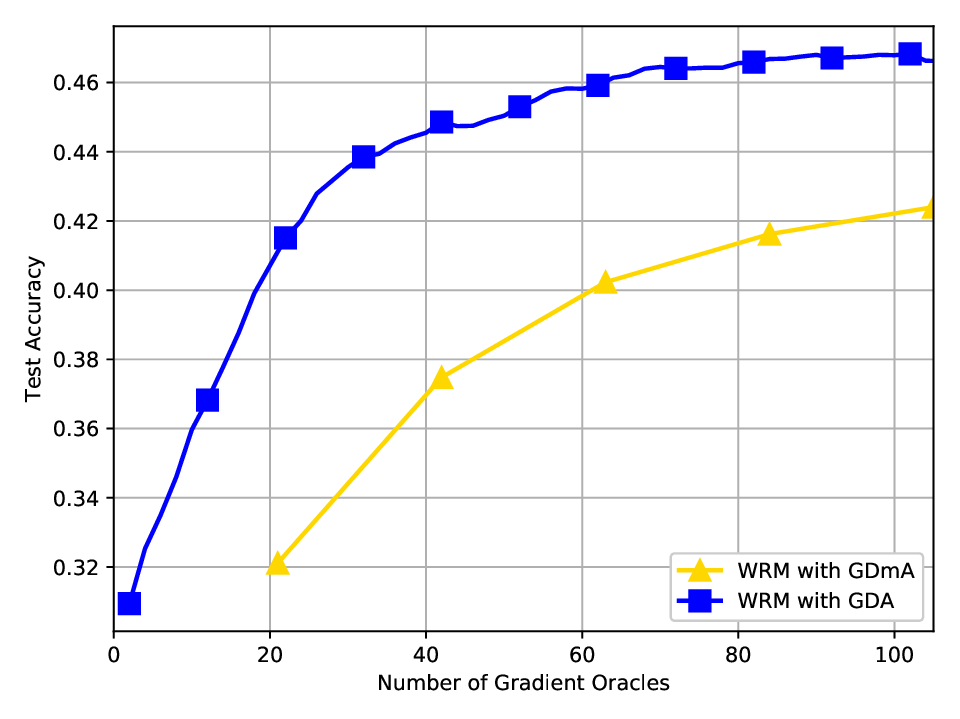}
}
\caption{Performance of WRM with GDmax (i.e., GDmA) and GDA on MNIST, Fashion-MNIST and CIFAR-10 datasets. We demonstrate test classification accuracy vs. time for different WRM models with GDmax and GDA. Note that $\gamma = 1.3$.}\label{Fig:Small-Data}\vspace*{-1em}
\end{figure*}
\section{Experiments}
In this section, we present several empirical results to show that two-timescale GDA outperforms GDmax. The task is to train the empirical Wasserstein robustness model (WRM)~\citep{Sinha-2018-Certifiable} over a collection of data samples $\{\xi_i\}_{i=1}^N$ with $\ell_2$-norm attack and a penalty parameter $\gamma>0$. Formally, we have 
\begin{equation}\label{prob:WRM-surrogate}
\min_\x \max_{\{\y_i\}_{i=1}^N \subseteq \YCal} \frac{1}{N}\left[\sum_{i=1}^N \left(\ell(\x, \y_i) - \gamma\|\y_i - \xi_i\|^2\right)\right]. 
\end{equation}
As shown in~\citet{Sinha-2018-Certifiable}, we often choose $\gamma > 0$ sufficiently large such that $\ell(\x, \y_i) - \gamma\|\y_i - \xi_i\|^2$ is strongly concave. To this end, problem~\eqref{prob:WRM-surrogate} is a nonconvex-strongly-concave minimax problem.  

We mainly follow the setting of~\citet{Sinha-2018-Certifiable} and consider training a neural network classifier on three datasets\footnote{https://keras.io/datasets/}: MNIST, Fashion-MNIST, and CIFAR-10, with the default cross validation. The architecture consists of $8 \times 8$, $6 \times 6$ and $5 \times 5$ convolutional filter layers with ELU activations followed by a fully connected layer and softmax output. Small and large adversarial perturbation is set with $\gamma \in \{0.4, 1.3\}$ as the same as~\citet{Sinha-2018-Certifiable}. The GDmax approach is denoted as \textsf{GDmA} in which $\eta_\x = \eta_\y = 10^{-3}$ and each inner loop contains $20$ gradient ascent. Two-timescale GDA is denoted as \textsf{GDA} in which $\eta_\x = 5 \times 10^{-5}$ and $\eta_\y = 10^{-3}$. Figure~\ref{Fig:Large-Data} and~\ref{Fig:Small-Data} show that GDA consistently outperforms GDmax on all datasets. Compared to MNIST and Fashion-MNIST, the improvement on CIFAR-10 is more significant which is worthy further exploration in the future. 

\section{Conclusion}\label{sec:conclusion}
In this paper, we show that two-time-scale GDA and SGDA return an $\epsilon$-stationary point in $O(\kappa^2\epsilon^{-2})$ gradient evaluations and $O(\kappa^3\epsilon^{-4})$ stochastic gradient evaluations in the nonconvex-strongly-concave case, and $O(\epsilon^{-6})$ gradient evaluations and $O(\epsilon^{-8})$ stochastic gradient evaluations in the nonconvex-concave case. Thus, these two algorithms are provably efficient in these settings. Future work aim to derive a lower bound for the complexity first-order algorithms in nonconvex-concave minimax problems.

\section*{Acknowledgments}
We would like to thank three anonymous referees for constructive suggestions that improve the quality of this paper. This work was supported in part by the Mathematical Data Science program of the Office of Naval Research under grant number N00014-18-1-2764.

\bibliographystyle{plainnat}
\bibliography{ref}

\newpage

\appendix \onecolumn
\section{Proof of Technical Lemmas}
We provide the complete proofs for the lemmas in Section~\ref{sec:prelim} and Section~\ref{sec:results}. 
\subsection{Proof of Lemma~\ref{Lemma:nc-moreau-envelope}}
We provide a proof for an expanded version of Lemma~\ref{Lemma:nc-moreau-envelope}. 
\begin{lemma}\label{Lemma:nc-moreau-envelope-complete}
If $f$ is $\ell$-smooth and $\YCal$ is bounded, we have
\begin{enumerate}
\item $\Phi_{1/2\ell}(\x)$ and $\prox_{\Phi/2\ell}(\x)$ are well-defined for $\forall\x \in \br^m$. 
\item $\Phi(\prox_{\Phi/2\ell}(\x)) \leq \Phi(\x)$ for any $\x \in \br^m$. 
\item $\Phi_{1/2\ell}$ is differentiable with $\grad\Phi_{1/2\ell}(\x) = 2\ell(\x - \prox_{\Phi/2\ell}(\x))$. 
\item $\Phi_{1/2\ell}(\x') - \Phi_{1/2\ell}(\x) - (\x' - \x)^\top\nabla\Phi_{1/2\ell}(\x) \leq (\ell/2)\|\x' - \x\|^2$ for any $\x', \x \in \br^m$. 
\end{enumerate} 
\end{lemma}
\begin{proof}
By the definition of $\Phi$, we have 
\begin{equation*}
\Psi(\x) \doteq \Phi(\x) + \tfrac{\ell\|\x\|^2}{2} = \max_{\y \in \YCal} \left\{f(\x, \y) + \tfrac{\ell\|\x\|^2}{2}\right\}. 
\end{equation*}
Since $f$ is $\ell$-smooth, $f(\x, \y) + (\ell/2)\|\x\|^2$ is convex in $\x$ for any $\y \in \YCal$. Since $\YCal$ is convex and bounded, the Danskin's theorem~\citep{Rockafellar-2015-Convex} implies that $\Psi(\x)$ is convex. Putting these pieces yields that $\Phi(\w) + \ell\left\|\w - \x\right\|^2$ is $(\ell/2)$-strongly convex. This implies that $\Phi_{1/2\ell}(\x)$ and $\prox_{\Phi/2\ell}(\x)$ are well-defined. Furthermore, by the definition of $\prox_{\Phi/2\ell}(\x)$, we have
\begin{equation*}
\Phi(\prox_{\Phi/2\ell}(\x)) \ \leq \ \Phi_{1/2\ell}(\prox_{\Phi/2\ell}(\x)) \ \leq \ \Phi(\x), \quad \forall\x \in \br^m.  
\end{equation*}
Moreover,~\citet[Lemma~2.2]{Davis-2019-Stochastic} implies that $\Phi_{1/2\ell}$ is differentiable with 
\begin{equation*}
\grad\Phi_{1/2\ell}(\x) = 2\ell(\x - \prox_{\Phi/2\ell}(\x)). 
\end{equation*}
Finally, it follows from~\citet[Theorem~2.1.5]{Nesterov-2013-Introductory} that $\Phi_{1/2\ell}$ satisfies the last inequality. 
\end{proof}

\subsection{Proof of Lemma~\ref{Lemma:nc-stationary}}
Denote $\hat{\x} := \prox_{\Phi/2\ell}(\x)$, we have $\grad\Phi_{1/2\ell}(\x) = 2\ell(\x - \hat{\x})$ (cf. Lemma~\ref{Lemma:nc-moreau-envelope}) and hence $\|\hat{\x} - \x\| = \|\grad\Phi_{1/2\ell}(\x)\|/2\ell$. Furthermore, the optimality condition for $\prox_{\Phi/2\ell}(\x)$ implies that $2\ell(\x - \hat{\x}) \in \partial \Phi(\hat{\x})$. Putting these pieces together yields that $\min_{\xi \in \subg\Phi(\hat{\x})} \|\xi\| \leq \|\grad\Phi_{1/2\ell}(\x)\|$.

\subsection{Proof of Lemma~\ref{Lemma:nsc-structure}}
Since $f(\x, \y)$ is strongly concave in $\y$ for each $\x \in \br^m$, a function $\y^\star(\cdot)$ is unique and well-defined. Then we claim that $\y^\star(\cdot)$ is $\kappa$-Lipschitz. Indeed, let $\x_1, \x_2 \in \br^m$, the optimality of $\y^\star(\x_1)$ and $\y^\star(\x_2)$ implies that 
\begin{align}
(\y-\y^\star(\x_1))^\top\grady f(\x_1, \y^\star(\x_1)) & \leq 0, \qquad \forall \y \in \YCal, \label{inequality-Lipschitz-first} \\
(\y-\y^\star(\x_2))^\top\grady f(\x_2, \y^\star(\x_2)) & \leq 0, \qquad \forall \y \in \YCal. \label{inequality-Lipschitz-second}
\end{align}
Letting $\y=\y^\star(\x_2)$ in~\eqref{inequality-Lipschitz-first} and $\y=\y^\star(\x_1)$ in~\eqref{inequality-Lipschitz-second} and summing the resulting two inequalities yields
\begin{equation}\label{inequality-Lipschitz-third}
(\y^\star(\x_2)-\y^\star(\x_1))^\top(\grady f(\x_1, \y^\star(\x_1)) - \grady f(\x_2, \y^\star(\x_2))) \leq 0. 
\end{equation} 
Recall that $f(\x_1, \cdot)$ is $\mu$-strongly concave, we have
\begin{equation}\label{inequality-Lipschitz-fourth}
\left(\y^\star(\x_2) - \y^\star(\x_1)\right)^\top\left(\grady f(\x_1, \y^\star(\x_2)) - \grady f(\x_1, \y^\star(\x_1))\right) + \mu \left\|\y^\star(\x_2) - \y^\star(\x_1)\right\|^2 \leq 0. 
\end{equation}
Then we conclude the desired result by combining~\eqref{inequality-Lipschitz-third} and~\eqref{inequality-Lipschitz-fourth} with $\ell$-smoothness of $f$, i.e., 
\begin{eqnarray*}
\mu \|\y^\star(\x_2) - \y^\star(\x_1)\|^2 & \leq & (\y^\star(\x_2) - \y^\star(\x_1))^\top(\grady f(\x_2, \y^\star(\x_2)) - \grady f(\x_1, \y^\star(\x_2))) \\
& \leq & \ell\|\y^\star(\x_2) - \y^\star(\x_1)\|\|\x_2 - \x_1\|. 
\end{eqnarray*}
Since $\y^\star(\x)$ is unique and $\YCal$ is convex and bounded, we conclude from Danskin's theorem~\citep{Rockafellar-2015-Convex} that $\Phi$ is differentiable with $\grad \Phi(\x) = \gradx f\left(\x, \y^\star(\x)\right)$. Since $\grad \Phi(\x) = \gradx f\left(\x, \y^\star(\x)\right)$, we have
\begin{equation*}
\|\grad \Phi(\x) - \grad \Phi(\x')\| = \|\gradx f(\x, \y^\star(\x)) - \gradx f(\x', \y^\star(\x'))\| \leq \ell(\|\x - \x'\| + \|\y^\star(\x) - \y^\star(\x')\|). 
\end{equation*}
Since $\y^\star(\cdot)$ is $\kappa$-Lipschitz, we conclude the desired result by plugging $\|\y^\star(\x) - \y^\star(\x')\| \leq \kappa$. Since $\kappa \geq 1$, $\Phi$ is $2\kappa\ell$-smooth. The last inequality follows from~\citet[Theorem~2.1.5]{Nesterov-2013-Introductory}. 

\subsection{Proof of Lemma~\ref{Lemma:nc-structure}}
We have $\Phi$ is $\ell$-weakly convex and $\partial\Phi(\x) = \partial\Psi(\x) - \ell\x$ where $\Psi(\x) = \max_{\y \in \YCal} \{f(\x, \y) + (\ell/2)\|\x\|^2\}$. Since $f(\x, \y) + (\ell/2)\|\x\|^2$ is convex in $\x$ for each $\y \in \YCal$ and $\YCal$ is bounded, Danskin's theorem implies that $\gradx f(\x, \y^\star(\x)) + \ell\x \in \partial\Psi(\x)$. Putting these pieces together yields that $\gradx f(\x, \y^\star(\x)) \in \partial\Phi(\x)$. 

\subsection{Proof of Lemma on Stochastic Gradient}
The following lemma establishes some properties of the stochastic gradients sampled at each iteration. 
\begin{lemma}\label{Lemma:SG-unbiased}
$\frac{1}{M}\sum_{i=1}^M G_\x(\x_t, \y_t, \xi_i)$ and $\frac{1}{M}\sum_{i=1}^M G_\y(\x_t, \y_t, \xi_i)$ are unbiased and have bounded variance,
\begin{equation*}
\begin{array}{ll}
\EE\left[\tfrac{1}{M}\sum_{i=1}^M G_\x(\x_t, \y_t, \xi_i)\right] = \gradx f(\x_t, \y_t), & \ \EE\left[\left\|\tfrac{1}{M}\sum_{i=1}^M G_\x(\x_t, \y_t, \xi_i)\right\|^2\right] \leq \|\gradx f(\x_t, \y_t)\|^2 + \tfrac{\sigma^2}{M}, \\
\EE\left[\tfrac{1}{M}\sum_{i=1}^M G_\y(\x_t, \y_t, \xi_i)\right] = \grady f(\x_t, \y_t), & \ \EE\left[\left\|\tfrac{1}{M}\sum_{i=1}^M G_\y(\x_t, \y_t, \xi_i)\right\|^2\right] \leq \|\grady f(\x_t, \y_t)\|^2 + \tfrac{\sigma^2}{M}. 
\end{array}
\end{equation*}
\end{lemma}
\begin{proof}
Since $G = (G_\x, G_\y)$ is unbiased, we have 
\begin{equation*}
\EE\left[\tfrac{1}{M}\sum_{i=1}^M G_\x(\x_t, \y_t, \xi_i)\right] = \gradx f(\x_t, \y_t), \quad \EE\left[\tfrac{1}{M}\sum_{i=1}^M G_\y(\x_t, \y_t, \xi_i)\right] = \grady f(\x_t, \y_t). 
\end{equation*}
Furthermore, we have
\begin{eqnarray*}
\EE\left[\left\|\tfrac{1}{M}\sum_{i=1}^M G_\x(\x_t, \y_t, \xi_i) - \gradx f(\x_t, \y_t)\right\|^2\right] & = & \tfrac{\sum_{i=1}^M \EE[\|G_\x(\x_t, \y_t, \xi_i) - \gradx f(\x_t, \y_t)\|^2]}{M^2} \leq \tfrac{\sigma^2}{M}, \\
\EE\left[\left\|\tfrac{1}{M}\sum_{i=1}^M G_\y(\x_t, \y_t, \xi_i) - \grady f(\x_t, \y_t)\right\|^2\right] & = & \tfrac{\sum_{i=1}^M \EE[\|G_\y(\x_t, \y_t, \xi_i) - \grady f(\x_t, \y_t)\|^2]}{M^2} \leq \tfrac{\sigma^2}{M}. 
\end{eqnarray*}
Putting these pieces together yields the desired result. 
\end{proof}

\section{Proof for Propositions~\ref{prop:criterion-nsc} and~\ref{prop:criterion-nc}}
\label{sec:opt}
We provide the detailed proof of Propositions~\ref{prop:criterion-nsc} and~\ref{prop:criterion-nc}. 
\paragraph{Proof of Proposition~\ref{prop:criterion-nsc}:} Assume that a point $\hat{\x}$ satisfies that $\|\grad\Phi(\hat{\x})\| \leq \epsilon$, the optimization problem $\max_{\y \in \YCal} f(\hat{\x}, \y)$ is strongly concave (cf.\ Assumption \ref{Assumption:nsc}) and $\y^\star(\hat{\x})$ is uniquely defined. We apply gradient descent for solving such problem and obtain a point $\y' \in \YCal$ satisfying that 
\begin{equation*}
\y^+ = \proj_\YCal(\y'+(1/\ell)\grady f(\hat{\x}, \y')), \quad \|\y^+ - \y'\| \leq \tfrac{\epsilon}{\ell}, \quad \|\y^+ - \y^\star(\hat{\x})\| \leq \epsilon. 
\end{equation*}
If $\|\grad\Phi(\hat{\x})\| \leq \epsilon$, we have 
\begin{equation*}
\|\gradx f(\hat{\x}, \y^+)\| \leq \|\gradx f(\hat{\x}, \y^+) - \grad \Phi(\hat{\x})\| + \|\grad \Phi(\hat{\x})\| = \|\gradx f(\hat{\x}, \y^+) - \gradx f(\hat{\x}, \y^\star(\hat{\x}))\| + \epsilon. 
\end{equation*}
Since $f(\cdot, \cdot)$ is $\ell$-smooth, we have
\begin{equation*}
\|\gradx f(\hat{\x}, \y^+)\| \leq \ell\|\y^+ - \y^\star(\hat{\x})|\| + \epsilon = O(\epsilon). 
\end{equation*}
The required number of gradient evaluations is $\OCal(\kappa \log(1/\epsilon))$. This argument holds for applying stochastic gradient with proper stepsize and the required number of stochastic gradient evaluations is $\OCal(1/\epsilon^2)$. 

Conversely, if a point $(\hat{\x}, \hat{\y})$ satisfies that 
\begin{equation*}
\|\gradx f(\hat{\x}, \hat{\y}^+)\| \leq \tfrac{\epsilon}{\kappa}, \quad \|\hat{\y}^+ - \hat{\y}\| \leq \tfrac{\epsilon}{\kappa\ell}, 
\end{equation*}
where $\hat{\y}^+ = \proj_\YCal(\hat{\y} + (1/\ell) \grady f(\hat{\x}, \hat{\y}))$. Then, we have 
\begin{equation*}
\|\nabla \Phi(\hat{\x})\| \leq \|\nabla \Phi(\hat{\x}) - \nabla_\x f(\hat{\x}, \hat{\y}^+)\| + \|\nabla_\x f(\hat{\x}, \hat{\y}^+)\| \leq \ell\|\hat{\y}^+ - \y^\star(\hat{\x})\| + \tfrac{\epsilon}{\kappa}. 
\end{equation*}
Since $f(\hat{\x}, \cdot)$ is $\mu$-strongly-concave over $\YCal$, the error bound condition~\citep{Drusvyatskiy-2018-Error} holds here and we have
\begin{equation*}
\|\hat{\y}^+ - \y^\star(\hat{\x})\| \leq \|\hat{\y} - \y^\star(\hat{\x})\| \leq \kappa\|\proj_\YCal(\hat{\y} + (1/\ell) \grady f(\hat{\x}, \hat{\y})) - \hat{\y}\| \leq \tfrac{\epsilon}{\ell}.  
\end{equation*}
Therefore, we conclude that 
\begin{equation*}
\|\nabla \Phi(\hat{\x})\| \leq \epsilon + \tfrac{\epsilon}{\kappa} = O(\epsilon). 
\end{equation*}
This completes the proof. 

\subsection{Proof of Proposition~\ref{prop:criterion-nc}}
Suppose that $\hat{\x}$ satisfies that $\|\grad\Phi_{1/2\ell}(\hat{\x})\| \leq\epsilon$, the function $f(\x, \y) + \ell\|\x - \hat{\x}\|^2$ is strongly convex in $\x$ and concave in $\y$ (cf. Assumption \ref{Assumption:nc}) and $\x^\star(\hat{\x}) = \argmin_{\x \in \br^m} \Phi(\x) + \ell\|\x - \hat{\x}\|^2$ is uniquely defined. We apply extragradient algorithm for solving such problem and obtain a point $(\x', \y')$ satisfying that 
\begin{equation*}
\|\nabla_\x f(\x', \y^+) + 2\ell(\x' - \hat{\x})\| \leq \epsilon, \quad \|\y^+ - \y'\| \leq \tfrac{\epsilon}{\ell}, \quad \|\x' - \x^\star(\hat{\x})\| \leq \tfrac{\epsilon}{\ell}. 
\end{equation*}
where $\y^+ = \proj_\YCal(\y'+(1/\ell)\grady f(\x', \y'))$. Since $2\ell\|\x^\star(\hat{\x}) - \hat{\x}\| = \|\nabla \Phi_{1/2\ell}(\hat{\x})\| \leq \epsilon$, we have 
\begin{eqnarray*}
\|\nabla_\x f(\x', \y^+)\| & \leq & \|\nabla_\x f(\x', \y^+) + 2\ell(\x' - \hat{\x})\| + 2\ell\|\x' - \hat{\x}\| \leq  \epsilon + 2\ell\|\x' - \x^\star(\hat{\x})\| + 2\ell\|\x^\star(\hat{\x}) - \hat{\x}\| \\
& \leq & 3\epsilon + \epsilon = O(\epsilon). 
\end{eqnarray*}
The required number of gradient evaluations is $O(\epsilon^{-2})$~\citep{Mokhtari-2019-Proximal}. This argument holds for applying stochastic mirror-prox algorithm and the required number of stochastic gradient evaluations is $O(\epsilon^{-4})$~\citep{Juditsky-2011-Solving}. 

Conversely, we let $\hat{\y}^+ = \proj_\YCal(\hat{\y}+(1/\ell)\grady f(\hat{\x}, \hat{\y}))$ for simplicity. By definition, we have
\begin{equation*}
\|\grad \Phi_{1/2\ell}(\hat{\x})\|^2 = 4\ell^2\|\hat{\x} - \x^\star(\hat{\x})\|^2. 
\end{equation*}
Since $\Phi(\cdot) + \ell\|\cdot - \hat{\x}\|^2$ is $\ell/2$-strongly-convex, we have
\begin{eqnarray}\label{inequality:criterion-nc-first}
\lefteqn{\max_{\y \in \YCal} f(\hat{\x}, \y) - \max_{\y \in \YCal} f(\x^\star(\hat{\x}), \y) - \ell\|\hat{\x} - \x^\star(\hat{\x})\|^2} \\ 
& = & \Phi(\hat{\x}) - \Phi(\x^\star(\hat{\x})) - \ell\|\x^\star(\hat{\x}) - \hat{\x}\|^2 \geq \tfrac{\ell\|\hat{\x} - \x^\star(\hat{\x})\|^2}{4} = \tfrac{\|\grad \Phi_{1/2\ell}(\hat{\x})\|^2}{16\ell}. \nonumber
\end{eqnarray}
Furthermore, we have
\begin{eqnarray*}
\lefteqn{\max_{\y \in \YCal} f(\hat{\x}, \y) - \max_{\y \in \YCal} f(\x^\star(\hat{\x}), \y) - \ell\|\x^\star(\hat{\x}) - \hat{\x}\|^2} \\
& \leq & \max_{\y \in \YCal} f(\hat{\x}, \y) - f(\hat{\x}, \hat{\y}^+) + f(\hat{\x}, \hat{\y}^+) - \max_{\y \in \YCal} f(\x^\star(\hat{\x}), \y) - \ell\|\x^\star(\hat{\x}) - \hat{\x}\|^2 \\
& \leq & \max_{\y \in \YCal} f(\hat{\x}, \y) - f(\hat{\x}, \hat{\y}^+) + (f(\hat{\x}, \hat{\y}^+) - f(\x^\star(\hat{\x}), \hat{\y}^+) - \ell\|\x^\star(\hat{\x}) - \hat{\x}\|^2) \\ 
& \leq & \max_{\y \in \YCal} f(\hat{\x}, \y) - f(\hat{\x}, \hat{\y}^+) + (\|\hat{\x} - \x^\star(\hat{\x})\|\|\gradx f(\hat{\x}, \hat{\y}^+)\| - \ell\|\hat{\x} - \x^\star(\hat{\x})\|^2) \\
& \leq & \max_{\y \in \YCal} f(\hat{\x}, \y) - f(\hat{\x}, \hat{\y}^+) + \tfrac{\|\gradx f(\hat{\x}, \hat{\y}^+)\|^2}{4\ell}. 
\end{eqnarray*}
By the definition of $\hat{\y}^+$, we have
\begin{equation*}
(\y - \hat{\y}^+)^\top(\hat{\y}^+ - \hat{\y} - (1/\ell)\grady f(\hat{\x}, \hat{\y})) \geq 0, \quad \textnormal{for all } \y \in \YCal. 
\end{equation*}
Together with the $\ell$-smoothness of the function $f(\hat{\x}, \cdot)$ and the boundedness of $\YCal$, we have
\begin{equation*}
f(\hat{\x}, \y) - f(\hat{\x}, \hat{\y}^+) \leq \tfrac{\ell}{2}(\|\y - \hat{\y}\|^2 - \|\y - \hat{\y}^+\|^2) \leq \ell D\|\hat{\y}^+ - \hat{\y}\|, \quad \textnormal{for all } \y \in \YCal. 
\end{equation*}
Putting these pieces together yields that 
\begin{equation*}
\max_{\y \in \YCal} f(\hat{\x}, \y) - \max_{\y \in \YCal} f(\x^\star(\hat{\x}), \y) - \ell\|\x^\star(\hat{\x}) - \hat{\x}\|^2 \leq \ell D\|\hat{\y}^+ - \hat{\y}\| + \tfrac{\|\gradx f(\hat{\x}, \hat{\y}^+)\|^2}{4\ell}. 
\end{equation*}
Since a point $(\hat{\x}, \hat{\y})$ satisfies $\|\nabla_\x f(\hat{\x}, \hat{\y}^+)\| \leq \epsilon^2/(\ell D)$ and $\|\hat{\y}^+ - \hat{\y}\| \leq \epsilon^2/(\ell^2 D)$, we have
\begin{equation}\label{inequality:criterion-nc-second}
\max_{\y \in \YCal} f(\hat{\x}, \y) - \max_{\y \in \YCal} f(\x^\star(\hat{\x}), \y) - \ell\|\x^\star(\hat{\x}) - \hat{\x}\|^2 \leq \tfrac{\epsilon^2}{\ell} + \tfrac{\epsilon^4}{4\ell^3 D^2}. 
\end{equation}
Putting these pieces together yields that $\|\grad \Phi_{1/2\ell}(\hat{\x})\| = O(\epsilon)$. This completes the proof. 

\section{Proof of Theorems in Section~\ref{sec:results_sc}}
We first specify the choice of parameters in Theorem~\ref{Theorem:nsc-GDA-complexity-bound} and Theorem~\ref{Theorem:nsc-SGDA-complexity-bound}. Then, we present the proof of the main theorems in Section~\ref{sec:results_sc} with several technical lemmas. Note first that the case of $\ell D \lesssim \epsilon$ is trivial. Indeed, this means that the set $\YCal$ is sufficiently small such that a single gradient ascent step is enough for approaching the $\epsilon$-neighborhood of the optimal solution. In this case, the nonconvex-strongly-concave minimax problem reduces to a nonconvex smooth minimization problem, which has been studied extensively in the existing literature. 

\subsection{Choice of Parameters in Theorem~\ref{Theorem:nsc-GDA-complexity-bound} and~\ref{Theorem:nsc-SGDA-complexity-bound}}
In this subsection, we present the full version of Theorems~\ref{Theorem:nsc-GDA-complexity-bound} and~\ref{Theorem:nsc-SGDA-complexity-bound} with the detailed choice of $\eta_\x$, $\eta_\y$ and $M$ which are important to subsequent analysis.  
\begin{theorem}\label{Theorem:nsc-GDA-app}
Under Assumption~\ref{Assumption:nsc} and letting the step sizes $\eta_\x > 0$ and $\eta_\y > 0$ be chosen as $\eta_\x = 1/[16(\kappa + 1)^2\ell]$ and $\eta_\y = 1/\ell$, the iteration complexity of Algorithm \ref{Algorithm:GDA} to return an $\epsilon$-stationary point is bounded by 
\begin{equation*}
O\left(\frac{\kappa^2\ell\Delta_\Phi + \kappa\ell^2 D^2}{\epsilon^2}\right), 
\end{equation*}
which is also the total gradient complexity of the algorithm. 
\end{theorem} 
\begin{theorem}\label{Theorem:nsc-SGDA-app}
Under Assumptions~\ref{Assumption:stoc-oracle} and~\ref{Assumption:nsc} and letting the step sizes $\eta_\x > 0$ and $\eta_\y > 0$ be the same in Theorem~\ref{Theorem:nsc-GDA-complexity-bound} with the batch size $M = \max\{1, 48\kappa\sigma^2\epsilon^{-2}\}$, the number of iterations required by Algorithm \ref{Algorithm:SGDA} to return an $\epsilon$-stationary point is bounded by $O((\kappa^2\ell\Delta_\Phi + \kappa\ell^2 D^2)\epsilon^{-2})$ which gives the total gradient complexity of the algorithm: 
\begin{equation*}
O\left(\frac{\kappa^2\ell\Delta_\Phi + \kappa\ell^2 D^2}{\epsilon^2}\max\left\{1, \ \frac{\kappa\sigma^2}{\epsilon^2}\right\}\right). 
\end{equation*}
\end{theorem}
\subsection{Proof of Technical Lemmas}
In this subsection, we present three key lemmas which are important for the subsequent analysis. 
\begin{lemma}\label{Lemma:nsc-key-descent}
For two-timescale GDA, the iterates $\{\x_t\}_{t \geq 1}$ satisfies the following inequality, 
\begin{equation*}
\Phi(\x_t) \leq \Phi(\x_{t-1}) - \left(\tfrac{\eta_{\x}}{2} - 2\eta_{\x}^2\kappa\ell\right) \|\grad \Phi(\x_{t-1})\|^2 + \left(\tfrac{\eta_{\x}}{2} + 2\eta_{\x}^2\kappa\ell\right)\|\grad\Phi(\x_{t-1}) - \gradx f(\x_{t-1}, \y_{t-1})\|^2. 
\end{equation*}
For two-timescale SGDA, the iterates $\{\x_t\}_{t \geq 1}$ satisfy the following inequality:
\begin{eqnarray*}
\EE[\Phi(\x_t)] & \leq & \EE[\Phi(\x_{t-1})] - \left(\tfrac{\eta_{\x}}{2} - 2\eta_{\x}^2\kappa\ell\right) \EE[\|\grad \Phi(\x_{t-1})\|^2] \\
& & + \left(\tfrac{\eta_{\x}}{2} + 2\eta_{\x}^2\kappa\ell\right)\EE[\|\grad\Phi(\x_{t-1}) - \gradx f(\x_{t-1}, \y_{t-1})\|^2] + \tfrac{\eta_\x^2\kappa\ell\sigma^2}{M}.
\end{eqnarray*}
\end{lemma}
\begin{proof}
We first consider the deterministic setting. Since $\Phi$ is $(\ell + \kappa\ell)$-smooth, we have
\begin{equation}\label{nsc-descent-first}
\Phi(\x_t) - \Phi(\x_{t-1}) - (\x_t - \x_{t-1})^\top\nabla\Phi(\x_{t-1}) \leq \kappa\ell\|\x_t - \x_{t-1}\|^2. 
\end{equation}
Plugging $\x_t - \x_{t-1} = - \eta_\x\gradx f(\x_{t-1}, \y_{t-1})$ into~\eqref{nsc-descent-first} yields that
\begin{eqnarray}\label{nsc-descent-second}
\Phi(\x_t) & \leq & \Phi(\x_{t-1})  - \eta_\x\|\nabla\Phi(\x_{t-1})\|^2 + \eta_\x^2\kappa\ell\|\gradx f(\x_{t-1}, \y_{t-1})\|^2 \\
& & + \eta_\x(\nabla\Phi(\x_{t-1}) - \gradx f(\x_{t-1}, \y_{t-1}))^\top \nabla\Phi(\x_{t-1}). \nonumber
\end{eqnarray}
Using the Young's inequality, we have
\begin{equation}\label{nsc-descent-third}
(\nabla\Phi(\x_{t-1}) - \gradx f(\x_{t-1}, \y_{t-1}))^\top \nabla\Phi(\x_{t-1}) \leq \tfrac{\|\nabla\Phi(\x_{t-1}) - \gradx f(\x_{t-1}, \y_{t-1})\|^2 + \|\nabla\Phi(\x_{t-1})\|^2}{2}. 
\end{equation}
By the Cauchy-Schwartz inequality, we have
\begin{equation}\label{nsc-descent-fourth}
\|\gradx f(\x_{t-1}, \y_{t-1})\|^2 \leq 2\left(\|\nabla\Phi(\x_{t-1}) - \gradx f(\x_{t-1}, \y_{t-1})\|^2 + \|\nabla\Phi(\x_{t-1})\|^2\right).
\end{equation}
Plugging~\eqref{nsc-descent-third} and~\eqref{nsc-descent-fourth} into~\eqref{nsc-descent-second} yields the first desired inequality. 

We proceed to consider the stochastic setting. Plugging $\x_t - \x_{t-1} = - \eta_\x\left(\frac{1}{M}\sum_{i=1}^{M} G_\x(\x_{t-1}, \y_{t-1}, \xi_i)\right)$ into~\eqref{nsc-descent-first} yields that
\begin{eqnarray*}
\Phi(\x_t) & \leq & \Phi(\x_{t-1})  - \eta_\x\|\nabla\Phi(\x_{t-1})\|^2 + \eta_\x^2\kappa\ell\left\|\tfrac{1}{M}\sum_{i=1}^{M} G_\x(\x_{t-1}, \y_{t-1}, \xi_i)\right\|^2 \\
& & + \eta_\x\left(\nabla\Phi(\x_{t-1}) - \left(\tfrac{1}{M}\sum_{i=1}^{M} G_\x(\x_{t-1}, \y_{t-1}, \xi_i)\right)\right)^\top \nabla\Phi(\x_t).
\end{eqnarray*}
Taking the expectation on both sides, conditioned on $(\x_{t-1}, \y_{t-1})$, yields that 
\begin{eqnarray}\label{nsc-descent-fifth}
\lefteqn{\EE[\Phi(\x_t) \mid \x_{t-1}, \y_{t-1}] \leq \Phi(\x_{t-1})  - \eta_\x \|\nabla\Phi(\x_{t-1})\|^2 + \eta_\x^2\kappa\ell \|\gradx f(\x_{t-1}, \y_{t-1})\|^2} \\ 
& & + \eta_\x (\nabla\Phi(\x_{t-1}) - \gradx f(\x_{t-1}, \y_{t-1}))^\top \nabla\Phi(\x_{t-1}) + \eta_\x^2\kappa\ell \|\gradx f(\x_{t-1}, \y_{t-1})\|^2 \nonumber \\
& & + \eta_\x^2\kappa\ell \cdot \EE\left[\left\|\tfrac{1}{M}\sum_{i=1}^{M} G_\x(\x_{t-1}, \y_{t-1}, \xi_i) - \gradx f(\x_{t-1}, \y_{t-1})\right\|^2 \mid \x_{t-1}, \y_{t-1}\right]. \nonumber
\end{eqnarray}
Plugging~\eqref{nsc-descent-third} and~\eqref{nsc-descent-fourth} into~\eqref{nsc-descent-fifth} and taking the expectation yields the second desired inequality.
\end{proof}
\begin{lemma}\label{Lemma:nsc-key-neighborhood}
For two-timescale GDA, let $\delta_t = \|\y^\star(\x_t) - \y_t\|^2$, the following statement holds true, 
\begin{equation*}
\delta_t \leq \left(1-\tfrac{1}{2\kappa} + 4\kappa^3\ell^2\eta_\x^2\right)\delta_{t-1} + 4\kappa^3\eta_\x^2 \|\grad\Phi(\x_{t-1})\|^2.  
\end{equation*}
For two-timescale SGDA, let $\delta_t = \EE[\|\y^\star(\x_t) - \y_t\|^2]$, the following statement holds true,
\begin{equation*}
\delta_t \leq \left(1-\tfrac{1}{2\kappa} + 8\kappa^3\ell^2\eta_\x^2\right)\delta_{t-1} + 8\kappa^3\eta_\x^2\EE[\|\grad\Phi(\x_{t-1})\|^2] + \tfrac{4\sigma^2\kappa^3\eta_\x^2}{M} + \tfrac{2\sigma^2}{\ell^2 M}.   
\end{equation*}
\end{lemma}
\begin{proof}
We prove the deterministic setting. Since $f(\x_t, \cdot)$ is $\mu$-strongly concave and $\eta_\y = 1/\ell$, we have
\begin{equation}\label{nsc-neighborhood-first}
\|\y^\star(\x_{t-1}) - \y_t\|^2 \leq \left(1 - \tfrac{1}{\kappa}\right)\delta_{t-1}. 
\end{equation}
Using the Young's inequality, we have
\begin{eqnarray*}
\delta_t & \leq & \left(1 + \tfrac{1}{2(\kappa - 1)}\right)\|\y^\star(\x_{t-1}) - \y_t\|^2 + (1 + 2(\kappa - 1))\|\y^\star(\x_t) - \y^\star(\x_{t-1})\|^2 \\
& \leq & \left(\tfrac{2\kappa-1}{2\kappa-2}\right)\|\y^\star(\x_{t-1}) - \y_t\|^2 + 2\kappa\|\y^\star(\x_t) - \y^\star(\x_{t-1})\|^2 \\ 
& \overset{~\eqref{nsc-neighborhood-first}}{\leq} & \left(1-\tfrac{1}{2\kappa}\right)\delta_{t-1} + 2\kappa\|\y^\star(\x_t) - \y^\star(\x_{t-1})\|^2.
\end{eqnarray*}
Since $\y^\star(\cdot)$ is $\kappa$-Lipschitz, $\|\y^\star(\x_t) - \y^\star(\x_{t-1})\| \leq \kappa \|\x_t - \x_{t-1}\|$. Furthermore, we have
\begin{equation*}
\|\x_t - \x_{t-1}\|^2 = \eta_\x^2\|\gradx f(\x_{t-1}, \y_{t-1})\|^2 \leq 2\eta_\x^2\ell^2 \delta_{t-1} + 2\eta_\x^2\|\grad \Phi(\x_{t-1})\|^2. 
\end{equation*}
Putting these pieces together yields the first desired inequality. 

We then prove the stochastic setting. Since $f(\x_t, \cdot)$ is $\mu$-strongly concave and $\eta_\y = 1/\ell$, we have
\begin{equation}\label{nsc-neighborhood-second}
\EE[\|\y^\star(\x_{t-1}) - \y_t\|^2] \leq \left(1 - \tfrac{1}{\kappa}\right)\delta_{t-1} + \tfrac{\sigma^2}{\ell^2 M}. 
\end{equation}
By the Young's inequality, we have
\begin{eqnarray*}
\delta_t & \leq & \left(1 + \tfrac{1}{2(\max\{\kappa, 2\} - 1)}\right)\EE[\| \y^\star(\x_{t-1}) - \y_t\|^2] + (1 + 2(\max\{\kappa, 2\} - 1))\EE[\|\y^\star(\x_t) - \y^\star(\x_{t-1})\|^2] \\
& \leq & \left(\tfrac{2\max\{\kappa, 2\} - 1}{2\max\{\kappa, 2\} - 2}\right)\EE[\|\y^\star(\x_{t-1}) - \y_t\|^2] + 4\kappa \EE[\|\y^\star(\x_t) - \y^\star(\x_{t-1})\|^2] \\ 
& \overset{~\eqref{nsc-neighborhood-second}}{\leq} & \left(1 - \tfrac{1}{2\kappa}\right)\delta_{t-1} + 4\kappa\EE[\|\y^\star(\x_t) - \y^\star(\x_{t-1})\|^2] + \tfrac{2\sigma^2}{\ell^2 M}.
\end{eqnarray*}
Since $\y^\star(\cdot)$ is $\kappa$-Lipschitz, $\|\y^\star(\x_t) - \y^\star(\x_{t-1})\| \leq \kappa \|\x_t - \x_{t-1}\|$. Furthermore, we have
\begin{equation*}
\EE[\|\x_t - \x_{t-1}\|^2] = \eta_\x^2 \EE\left[\left\|\tfrac{1}{M}\sum_{i=1}^M G_\x(\x_{t-1}, \y_{t-1}, \xi_i)\right\|^2\right] \leq 2\eta_\x^2\ell^2 \delta_{t-1} + 2\eta_\x^2\EE[\|\grad \Phi(\x_{t-1})\|^2] + \tfrac{\eta_\x^2\sigma^2}{M}.  
\end{equation*}
Putting these pieces together yields the second desired inequality. 
\end{proof}
\begin{lemma}\label{Lemma:nsc-key-objective}
For two-timescale GDA, let $\delta_t = \|\y^\star(\x_t) - \y_t\|^2$, the following statement holds true,
\begin{equation*}
\Phi(\x_t) \leq \Phi(\x_{t-1}) - \tfrac{7\eta_\x}{16} \left\|\grad \Phi(\x_{t-1})\right\|^2 + \tfrac{9\eta_{\x}\ell^2\delta_{t-1}}{16}. 
\end{equation*}
For two-timescale SGDA, let $\delta_t = \EE[\|\y^\star(\x_t) - \y_t\|^2]$, the following statement holds true,
\begin{equation*}
\EE[\Phi(\x_t)] \leq \EE[\Phi(\x_{t-1})] - \tfrac{7\eta_\x}{16} \EE[\|\grad \Phi(\x_{t-1})\|^2] + \tfrac{9\eta_\x\ell^2\delta_{t-1}}{16} + \tfrac{\eta_\x^2\kappa\ell\sigma^2}{M}. 
\end{equation*}
\end{lemma}
\begin{proof}
For two-timescale GDA and SGDA, $\eta_\x = 1/16(\kappa+1)\ell$ and hence
\begin{equation}\label{nsc-descent-stepsize}
\tfrac{7\eta_{\x}}{16} \leq \tfrac{\eta_{\x}}{2} - 2\eta_{\x}^2\kappa\ell \leq \tfrac{\eta_{\x}}{2} + 2\eta_{\x}^2\kappa\ell \leq \tfrac{9\eta_{\x}}{16}.  
\end{equation}
Combining~\eqref{nsc-descent-stepsize} with the first inequality in Lemma~\ref{Lemma:nsc-key-descent} yields that
\begin{equation*}
\Phi(\x_t) \leq \Phi(\x_{t-1}) - \tfrac{7\eta_\x}{16} \|\grad \Phi(\x_{t-1})\|^2 + \tfrac{9\eta_\x}{16}\|\grad\Phi(\x_{t-1}) - \gradx f(\x_{t-1}, \y_{t-1})\|^2. 
\end{equation*}
Since $\grad \Phi(\x_{t-1}) = \gradx f(\x_{t-1}, \y^\star(\x_{t-1}))$, we have
\begin{equation*}
\|\grad\Phi(\x_{t-1}) - \gradx f(\x_{t-1}, \y_{t-1})\|^2 \leq \ell^2\|\y^\star(\x_{t-1}) - \y_{t-1}\|^2 =  \ell^2\delta_{t-1}. 
\end{equation*}
Putting these pieces together yields the first desired inequality. 

Combining~\eqref{nsc-descent-stepsize} with the second inequality in Lemma~\ref{Lemma:nsc-key-descent} yields that
\begin{equation*}
\EE[\Phi(\x_t)] \leq \EE[\Phi(\x_{t-1})] - \tfrac{7\eta_\x}{16} \EE[\|\grad \Phi(\x_{t-1})\|^2] + \tfrac{9\eta_\x}{16}\EE[\|\grad\Phi(\x_{t-1}) - \gradx f(\x_{t-1}, \y_{t-1})\|^2] + \tfrac{\eta_\x^2\kappa\ell\sigma^2}{M}. 
\end{equation*}
Since $\grad \Phi(\x_{t-1}) = \gradx f\left(\x_{t-1}, \y^\star(\x_{t-1})\right)$, we have
\begin{equation*}
\EE[\|\grad\Phi(\x_{t-1}) - \gradx f(\x_{t-1}, \y_{t-1})\|^2] \leq \ell^2\EE[\|\y^\star(\x_{t-1}) - \y_{t-1}\|^2] =  \ell^2\delta_{t-1}. 
\end{equation*}
Putting these pieces together yields the second desired inequality.
\end{proof}

\subsection{Proof of Theorem~\ref{Theorem:nsc-GDA-app}}
Throughout this subsection, we define $\gamma = 1 - 1/2\kappa + 4\kappa^3\ell^2\eta_\x^2$. Performing the first inequality in Lemma~\ref{Lemma:nsc-key-neighborhood} recursively yields that 
\begin{equation}\label{nsc-GDA-descent-residue}
\delta_t \leq \gamma^t \delta_0 + 4\kappa^3\eta_\x^2\left(\sum_{j=0}^{t-1} \gamma^{t-1-j} \|\grad \Phi(\x_j)\|^2 \right) \leq \gamma^t D^2 + 4\kappa^3\eta_\x^2\left(\sum_{j=0}^{t-1} \gamma^{t-1-j} \|\grad \Phi(\x_j)\|^2 \right). 
\end{equation} 
Combining~\eqref{nsc-GDA-descent-residue} with the first inequality in Lemma~\ref{Lemma:nsc-key-objective} yields that,
\begin{equation}\label{nsc-GDA-descent-main}
\Phi(\x_t) \leq \Phi(\x_{t-1}) - \tfrac{7\eta_\x}{16} \|\grad \Phi(\x_{t-1})\|^2 + \tfrac{9\eta_\x\ell^2\gamma^{t-1} D^2}{16} + \tfrac{9\eta_\x^3\ell^2\kappa^3}{4}\left(\sum_{j=0}^{t-2} \gamma^{t-2-j} \|\grad \Phi(\x_j)\|^2 \right).
\end{equation}
Summing up~\eqref{nsc-GDA-descent-main} over $t=1, 2, \ldots, T+1$ and rearranging the terms yields that
\begin{equation*}
\Phi(\x_{T+1}) \leq \Phi(\x_0) - \tfrac{7\eta_{\x}}{16} \left(\sum_{t=0}^T \|\grad \Phi(\x_t)\|^2\right) + \tfrac{9\eta_{\x}\ell^2 D^2}{16}\left(\sum_{t=0}^T \gamma^t \right) + \tfrac{9\eta_\x^3\ell^2\kappa^3}{4}\left(\sum_{t=1}^{T+1}\sum_{j=0}^{t-2} \gamma^{t-2-j} \|\grad \Phi(\x_j)\|^2 \right).
\end{equation*}
Since $\eta_{\x} = 1/16(\kappa+1)^2\ell$, we have $\gamma \leq 1 - \frac{1}{4\kappa}$ and $\frac{9\eta_\x^3\ell^2\kappa^3}{4} \leq \frac{9\eta_\x}{1024\kappa}$. This implies that $\sum_{t=0}^T \gamma^t \leq 4\kappa$ and
\begin{equation*}
\sum_{t=1}^{T+1}\sum_{j=0}^{t-2} \gamma^{t-2-j} \|\grad \Phi(\x_j)\|^2 \leq 4\kappa\left(\sum_{t=0}^T \|\grad \Phi(\x_t)\|^2\right) 
\end{equation*}
Putting these pieces together yields that 
\begin{equation*}
\Phi(\x_{T+1}) \leq \Phi(\x_0) - \tfrac{103\eta_\x}{256}\left(\sum_{t=0}^T \|\grad \Phi(\x_t)\|^2\right) + \tfrac{9\eta_{\x}\kappa\ell^2 D^2}{4}. 
\end{equation*}
By the definition of $\Delta_\Phi$, we have
\begin{equation*}
\tfrac{1}{T+1}\left(\sum_{t=0}^T \|\grad \Phi(\x_t)\|^2\right) \leq \tfrac{256(\Phi(\x_0) - \Phi(\x_{T+1}))}{103\eta_\x(T+1)} + \tfrac{576\kappa\ell^2D^2}{103(T+1)} \leq \tfrac{128\kappa^2 \ell \Delta_\Phi + 5\kappa\ell^2 D^2}{T+1}. 
\end{equation*}
This implies that the number of iterations required by Algorithm \ref{Algorithm:GDA} to return an $\epsilon$-stationary point is bounded by 
\begin{equation*}
O\left(\frac{\kappa^2\ell\Delta_\Phi + \kappa\ell^2 D^2}{\epsilon^2}\right),  
\end{equation*}
which gives the same total gradient complexity. 

\subsection{Proof of Theorem~\ref{Theorem:nsc-SGDA-app}}
Throughout this subsection, we define $\gamma = 1 - 1/2\kappa + 8\kappa^3\ell^2\eta_\x^2$. Performing the second inequality in Lemma~\ref{Lemma:nsc-key-neighborhood} recursively together with $\delta_0 \leq D^2$ yields that 
\begin{equation}\label{nsc-SGDA-descent-residue}
\delta_t \leq \gamma^t D^2 + 8\kappa^3\eta_\x^2\left(\sum_{j=0}^{t-1} \gamma^{t-1-j} \EE[\|\grad \Phi(\x_j)\|^2]\right) + \left(\tfrac{4\sigma^2\kappa^3\eta_\x^2}{M} + \tfrac{2\sigma^2}{\ell^2 M}\right)\left(\sum_{j=0}^{t-1} \gamma^{t-1-j} \right).
\end{equation} 
Combining~\eqref{nsc-SGDA-descent-residue} with the second inequality in Lemma~\ref{Lemma:nsc-key-objective} yields that,
\begin{eqnarray}\label{nsc-SGDA-descent-main}
\lefteqn{\EE[\Phi(\x_t)] \leq \EE[\Phi(\x_{t-1})] - \tfrac{7\eta_\x}{16} \EE[\|\grad \Phi(\x_{t-1})\|^2] + \tfrac{9\eta_\x\ell^2\gamma^{t-1} D^2}{16} + \tfrac{\eta_\x^2\kappa\ell\sigma^2}{M}} \\ 
& & + \tfrac{9\eta_\x^3\ell^2\kappa^3}{2}\left(\sum_{j=0}^{t-2} \gamma^{t-2-j} \EE[\|\grad \Phi(\x_j)\|^2]\right) + \tfrac{9\eta_\x\ell^2}{8}\left(\tfrac{2\sigma^2\kappa^3\eta_\x^2}{M} + \tfrac{\sigma^2}{\ell^2 M}\right)\left(\sum_{j=0}^{t-2} \gamma^{t-2-j}\right). \nonumber
\end{eqnarray}
Summing up~\eqref{nsc-SGDA-descent-main} over $t = 1, 2, \ldots, T+1$ and rearranging the terms yields that 
\begin{eqnarray*}
\lefteqn{\EE[\Phi(\x_{T+1})] \leq \Phi(\x_0) - \tfrac{7\eta_{\x}}{16} \left(\sum_{t=0}^T \EE[\|\grad \Phi(\x_t)\|^2]\right) + \tfrac{9\eta_{\x}\ell^2 D^2}{16}\left(\sum_{t=0}^T \gamma^t \right)} \\ 
& & + \tfrac{\eta_\x^2\kappa\ell\sigma^2(T + 1)}{M} + \tfrac{9\eta_\x^3\ell^2\kappa^3}{2}\left(\sum_{t=1}^{T+1}\sum_{j=0}^{t-2} \gamma^{t-2-j} \EE[\|\grad \Phi(\x_j)\|^2] \right) \nonumber \\
& & + \tfrac{9\eta_\x\ell^2}{8}\left(\tfrac{2\sigma^2\kappa^3\eta_\x^2}{M} + \tfrac{\sigma^2}{\ell^2 M}\right)\left(\sum_{t=1}^{T+1}\sum_{j=0}^{t-2} \gamma^{t-2-j} \right).
\end{eqnarray*}
Since $\eta_{\x} = 1/16(\kappa+1)^2\ell$, we have $\gamma \leq 1 - \frac{1}{4\kappa}$ and $\frac{9\eta_\x^3\ell^2\kappa^3}{2} \leq \frac{9\eta_\x}{1024\kappa}$ and $\frac{2\sigma^2\kappa^3\eta_\x^2}{M} \leq \frac{\sigma^2}{\ell^2M}$. This implies that $\sum_{t=0}^T \gamma^t \leq 4\kappa$ and
\begin{eqnarray*}
\sum_{t=1}^{T+1}\sum_{j=0}^{t-2} \gamma^{t-2-j} \EE[\|\grad \Phi(\x_j)\|^2] & \leq & 4\kappa\left(\sum_{t=0}^T \EE[\|\grad \Phi(\x_t)\|^2]\right), \\ 
\sum_{t=1}^{T+1} \sum_{j=0}^{t-2} \gamma^{t-1-j} & \leq & 4\kappa(T+1). 
\end{eqnarray*}
Putting these pieces together yields that 
\begin{equation*}
\EE[\Phi(\x_{T+1})] \leq \Phi(\x_0) - \tfrac{103\eta_\x}{256}\left(\sum_{t=0}^T \EE[\|\grad \Phi(\x_t)\|^2]\right) + \tfrac{9\eta_{\x}\kappa\ell^2 D^2}{4} + \tfrac{\eta_\x\sigma^2(T+1)}{16\kappa M} + \tfrac{9\eta_\x\kappa\sigma^2(T+1)}{M}.
\end{equation*}
By the definition of $\Delta_\Phi$, we have
\begin{eqnarray*}
\tfrac{1}{T+1}\left(\sum_{t=0}^T \EE[\|\grad \Phi(\x_t)\|^2]\right) & \leq & \tfrac{256(\Phi(\x_0) - \EE[\Phi(\x_{T+1})])}{103\eta_\x(T+1)} + \tfrac{576\kappa\ell^2D^2}{103(T+1)} + \tfrac{16\sigma^2}{103\kappa M} + \tfrac{2304\kappa\sigma^2}{103M}\\
& \leq & \tfrac{2\Delta_\Phi}{\eta_\x(T+1)} + \tfrac{5\kappa\ell^2 D^2}{T+1} + \tfrac{24\kappa\sigma^2}{M} \\
& \leq & \tfrac{128\kappa^2 \ell\Delta_\Phi + 5\kappa\ell^2 D^2}{T+1} + \tfrac{24\sigma^2\kappa}{M}. 
\end{eqnarray*}
This implies that the number of iterations required by Algorithm \ref{Algorithm:SGDA} to return an $\epsilon$-stationary point is bounded by 
\begin{equation*}
O\left(\frac{\kappa^2\ell\Delta_\Phi + \kappa\ell^2 D^2}{\epsilon^2}\right). 
\end{equation*} 
iterations, which gives the total gradient complexity of the algorithm:
\begin{equation*}
O\left(\frac{\kappa^2\ell\Delta_\Phi + \kappa\ell^2 D^2}{\epsilon^2}\max\left\{1, \ \frac{\kappa\sigma^2}{\epsilon^2}\right\}\right). 
\end{equation*}
This completes the proof. 

\section{Proof of Theorems in Section~\ref{sec:results_general}}
In this section, we first specify the choice of parameters in Theorems~\ref{Theorem:nc-GDA-complexity-bound} and~\ref{Theorem:nc-SGDA-complexity-bound}. Then we present the proof of main theorems in Section~\ref{sec:results_general} with several technical lemmas. Differently from the previous section, we include the case of $\ell D \lesssim \varepsilon$ in the analysis for nonconvex-concave minimax problems. 

\subsection{Choice of Parameters in Theorem~\ref{Theorem:nc-GDA-complexity-bound} and~\ref{Theorem:nc-SGDA-complexity-bound}}
In this subsection, we present the full version of Theorems~\ref{Theorem:nc-GDA-complexity-bound} and~\ref{Theorem:nc-SGDA-complexity-bound} with the detailed choice of $\eta_\x$, $\eta_\y$ and $M$ which are important to subsequent analysis.  
\begin{theorem}\label{Theorem:nc-GDA-app}
Under Assumption~\ref{Assumption:nc} and letting the step sizes $\eta_\x > 0$ and $\eta_\y > 0$ be chosen as $\eta_\x =  \min\{\epsilon^2/[16\ell L^2], \epsilon^4/[4096\ell^3 L^2 D^2]\}$ and $\eta_\y = 1/\ell$, the iterations complexity of Algorithm \ref{Algorithm:GDA} to return an $\epsilon$-stationary point is bounded by 
\begin{equation*}
O\left(\frac{\ell^3 L^2 D^2\widehat{\Delta}_\Phi}{\epsilon^6} + \frac{\ell^3 D^2\widehat{\Delta}_0}{\epsilon^4}\right).
\end{equation*}
which is also the total gradient complexity of the algorithm. 
\end{theorem}
\begin{theorem}\label{Theorem:nc-SGDA-app}
Under Assumptions~\ref{Assumption:stoc-oracle} and~\ref{Assumption:nc} and letting the step sizes $\eta_\x > 0$ and $\eta_\y > 0$ be chosen as $\eta_\x = \min\{\epsilon^2/[16\ell(L^2 + \sigma^2)], \epsilon^4/[8192\ell^3 D^2 L\sqrt{L^2+\sigma^2}], \  \epsilon^6/[65536\ell^3D^2\sigma^2L\sqrt{L^2+\sigma^2}]\}$ and $\eta_\y = \min\{1/2\ell, \epsilon^2/[16\ell\sigma^2]\}$ with a batch size $M = 1$, the iteration complexity of Algorithm \ref{Algorithm:SGDA} to return an $\epsilon$-stationary point is bounded by 
\begin{equation*}
O\left(\left(\frac{\ell^3 \left(L^2 + \sigma^2\right) D^2\widehat{\Delta}_\Phi}{\epsilon^6} + \frac{\ell^3 D^2\widehat{\Delta}_0}{\epsilon^4}\right)\max\left\{1, \ \frac{\sigma^2}{\epsilon^2}\right\}\right),  
\end{equation*}
which is also the total gradient complexity of the algorithm. 
\end{theorem}
\subsection{Proof of Technical Lemmas}
In this subsection, we present three key lemmas which are important for the subsequent analysis. 
\begin{lemma}\label{Lemma:nc-key-descent}
For two-timescale GDA, let $\Delta_t = \Phi(\x_t) - f(\x_t, \y_t)$, the following statement holds true, 
\begin{equation*}
\Phi_{1/2\ell}(\x_t) \leq \Phi_{1/2\ell}(\x_{t-1}) + 2\eta_\x\ell\Delta_{t-1} - \tfrac{\eta_\x}{4}\|\grad \Phi_{1/2\ell}(\x_{t-1})\|^2 + \eta_\x^2 \ell L^2. 
\end{equation*}
For two-timescale SGDA, let $\Delta_t = \EE\left[\Phi(\x_t) - f(\x_t, \y_t)\right]$, the following statement holds true, 
\begin{equation*}
\EE[\Phi_{1/2\ell}(\x_t)] \leq \EE[\Phi_{1/2\ell}(\x_{t-1})] + 2\eta_\x\ell\Delta_{t-1} - \tfrac{\eta_\x}{4}\EE[\|\grad \Phi_{1/2\ell}(\x_{t-1})\|^2] + \eta_\x^2 \ell(L^2 + \sigma^2).
\end{equation*}
\end{lemma}
\begin{proof}
We first consider the deterministic setting. Let $\hat{\x}_{t-1} = \prox_{\Phi/2\ell}(\x_{t-1})$, we have
\begin{equation}\label{nc-descent-first}
\Phi_{1/2\ell}(\x_t) \leq \Phi(\hat{\x}_{t-1}) + \ell\|\hat{\x}_{t-1} - \x_t\|^2. 
\end{equation}
Since $f(\cdot, \y)$ is $L$-Lipschitz for any $\y \in \YCal$, we have 
\begin{eqnarray}\label{nc-descent-second}
\lefteqn{\|\hat{\x}_{t-1} - \x_t\|^2 = \|\hat{\x}_{t-1} - \x_{t-1} + \eta_\x \gradx f(\x_{t-1}, \y_{t-1})\|^2} \\ 
& \leq & \|\hat{\x}_{t-1} - \x_{t-1}\|^2 + 2\eta_\x \langle \hat{\x}_{t-1} - \x_{t-1}, \gradx f(\x_{t-1}, \y_{t-1})\rangle + \eta_\x^2 L^2. \nonumber
\end{eqnarray}
Plugging~\eqref{nc-descent-second} into~\eqref{nc-descent-first} yields that 
\begin{equation*}
\Phi_{1/2\ell}(\x_t) \leq \Phi_{1/2\ell}(\x_{t-1}) + 2\eta_\x \ell \langle\hat{\x}_{t-1} - \x_{t-1}, \gradx f(\x_{t-1}, \y_{t-1})\rangle + \eta_\x^2 \ell L^2.
\end{equation*}
Since $f$ is $\ell$-smooth, we have
\begin{equation}\label{nc-descent-third}
\langle \hat{\x}_{t-1} - \x_{t-1}, \gradx f(\x_{t-1}, \y_{t-1})\rangle \leq f(\hat{\x}_{t-1}, \y_{t-1}) - f(\x_{t-1}, \y_{t-1}) + \tfrac{\ell}{2}\|\hat{\x}_{t-1} - \x_{t-1}\|^2. 
\end{equation}
Furthermore, $\Phi(\hat{\x}_{t-1}) \geq f(\hat{\x}_{t-1}, \y_{t-1})$. By the definition of $\Delta_t$, we have
\begin{equation}\label{nc-descent-fourth}
f(\hat{\x}_{t-1}, \y_{t-1}) - f(\x_{t-1}, \y_{t-1}) \leq \Phi(\hat{\x}_{t-1}) - f(\x_{t-1}, \y_{t-1}) \leq \Delta_{t-1} - \tfrac{\ell}{2}\|\hat{\x}_{t-1} - \x_{t-1}\|^2.
\end{equation}
Putting these pieces together with $\|\hat{\x}_{t-1} - \x_{t-1}\| = \|\grad\Phi_{1/2\ell}(\x_{t-1})\|/2\ell$ yields the first desired inequality. 

We proceed to the stochastic setting. Indeed, we have 
\begin{equation*}
\|\hat{\x}_{t-1} - \x_t\|^2 \leq \|\hat{\x}_{t-1} - \x_{t-1}\|^2 + \eta_\x^2\left\|\tfrac{1}{M}\sum_{i=1}^M G_\x(\x_{t-1}, \y_{t-1}, \xi_i)\right\| + \tfrac{2\eta_\x}{M}\sum_{i=1}^M \langle \hat{\x}_{t-1} - \x_{t-1},  G_\x(\x_{t-1}, \y_{t-1}, \xi_i)\rangle.
\end{equation*}
Taking the expectation of both sides, conditioned on $(\x_{t-1}, \y_{t-1})$, together with Lemma~\ref{Lemma:SG-unbiased} and the Lipschitz property of $f(\cdot, \y_{t-1})$ yields that 
\begin{eqnarray*}
\lefteqn{\EE[\|\hat{\x}_{t-1} - \x_t\|^2 \mid \x_{t-1}, \y_{t-1}] \leq \|\hat{\x}_{t-1} - \x_{t-1}\|^2 + 2\eta_\x \langle \hat{\x}_{t-1} - \x_{t-1}, \gradx f(\x_{t-1}, \y_{t-1})\rangle} \\
& & + \eta_\x^2 L^2 + \eta_\x^2\EE\left[\left\|\gradx f(\x_{t-1}, \y_{t-1}) - \tfrac{1}{M}\sum_{i=1}^M G_\x(\x_{t-1}, \y_{t-1}, \xi_i)\right\|^2 \mid \x_{t-1}, \y_{t-1} \right]. 
\end{eqnarray*}
Taking the expectation of both sides together with Lemma~\ref{Lemma:SG-unbiased} yields that 
\begin{equation*}
\EE[\|\hat{\x}_{t-1} - \x_t\|^2] \leq \EE[\|\hat{\x}_{t-1} - \x_{t-1}\|^2] + 2\eta_\x \EE[\langle \hat{\x}_{t-1} - \x_{t-1}, \gradx f(\x_{t-1}, \y_{t-1})\rangle] + \eta_\x^2(L^2 + \sigma^2). 
\end{equation*}
Combining with~\eqref{nc-descent-third} and~\eqref{nc-descent-fourth} yields that  
\begin{eqnarray*}
\EE[\Phi_{1/2\ell}(\x_t)] & \leq & \EE[\Phi_{1/2\ell}(\x_{t-1})] + 2\eta_\x \EE[\langle \hat{\x}_{t-1} - \x_{t-1}, \gradx f(\x_{t-1}, \y_{t-1})\rangle] + \eta_\x^2 \ell (L^2 + \sigma^2) \\
& \leq & \EE[\Phi_{1/2\ell}(\x_{t-1})] + 2\eta_\x\ell\Delta_{t-1} - \eta_\x \ell^2\EE[\|\hat{\x}_{t-1} - \x_{t-1}\|^2] + \eta_\x^2 \ell (L^2 + \sigma^2).
\end{eqnarray*}
This together with $\|\hat{\x}_{t-1} - \x_{t-1}\| = \|\grad\Phi_{1/2\ell}(\x_{t-1})\|/2\ell$ yields the second desired inequality. 
\end{proof}
\begin{lemma}\label{Lemma:nc-key-neighborhood}
For two-timescale GDA, let $\Delta_t = \Phi(\x_t) - f(\x_t, \y_t)$, the following statement holds true for $\forall s \leq t-1$, 
\begin{equation*}
\Delta_{t-1} \leq \eta_\x L^2 (2t-2s-1) + \tfrac{\ell}{2}(\|\y_{t-1} - \y^\star(\x_s)\|^2 - \|\y_t - \y^\star(\x_s)\|^2) + (f(\x_t, \y_t) - f(\x_{t-1}, \y_{t-1})). 
\end{equation*}
For two-timescale SGDA, let $\Delta_t = \EE[\Phi(\x_t) - f(\x_t, \y_t)]$, the following statement holds true for $\forall s \leq t-1$, 
\begin{eqnarray*}
\Delta_{t-1} & \leq & \eta_\x L\sqrt{L^2+\sigma^2}(2t-2s-1) + \tfrac{1}{2\eta_\y}(\EE[\|\y_{t-1} - \y^\star(\x_s)\|^2] - \EE[\|\y_t - \y^\star(\x_s)\|^2]) \\
& & + \EE[f(\x_t, \y_t) - f(\x_{t-1}, \y_{t-1})] + \tfrac{\eta_\y\sigma^2}{2}.    
\end{eqnarray*}
\end{lemma}
\begin{proof}
We consider the deterministic setting. For any $\y \in \YCal$, the convexity of $\YCal$ and the update of $\y_t$ imply that $(\y - \y_t)^\top (\y_t - \y_{t-1} - \eta_\y\grady f(\x_{t-1}, \y_{t-1})) \geq 0$. Rearranging the inequality yields
\begin{equation*}
\|\y - \y_t\|^2 \leq 2\eta_\y(\y_{t-1} - \y)^\top \grady f(\x_{t-1}, \y_{t-1}) + 2\eta_\y(\y_t - \y_{t-1})^\top \grady f(\x_{t-1}, \y_{t-1}) + \|\y - \y_{t-1}\|^2 - \|\y_t - \y_{t-1}\|^2. 
\end{equation*}
Since $f(\x_{t-1}, \cdot)$ is concave and $\ell$-smooth and $\eta_\y = 1/\ell$, we have
\begin{equation*}
f(\x_{t-1}, \y) - f(\x_{t-1}, \y_t) \leq \tfrac{\ell}{2}(\|\y - \y_{t-1}\|^2 - \|\y - \y_t\|^2). 
\end{equation*}
Plugging $\y = \y^\star(\x_s)$ ($s \leq t-1$) in the above inequality yields that 
\begin{equation*}
f(\x_{t-1}, \y^\star(\x_s)) - f(\x_{t-1}, \y_t) \leq \tfrac{\ell}{2}(\|\y_{t-1} - \y^\star(\x_s)\|^2 - \|\y_t - \y^\star(\x_s)\|^2). 
\end{equation*}
By the definition of $\Delta_{t-1}$, we have
\begin{eqnarray*}
\Delta_{t-1} & \leq & (f(\x_{t-1}, \y^\star(\x_{t-1})) - f(\x_{t-1}, \y^\star(\x_s))) + (f(\x_t, \y_t) - f(\x_{t-1}, \y_{t-1})) + (f(\x_{t-1}, \y_t) - f(\x_t, \y_t)) \\
& & + \tfrac{\ell}{2}(\|\y_{t-1} - \y^\star(\x_s)\|^2 - \|\y_t - \y^\star(\x_s)\|^2).
\end{eqnarray*}
Since $f(\x_s, \y^\star(\x_s)) \geq f(\x_s, \y)$ for $\forall\y \in \YCal$, we have
\begin{eqnarray}\label{nc-neighborhood-first}
\lefteqn{f(\x_{t-1}, \y^\star(\x_{t-1})) - f(\x_{t-1}, \y^\star(\x_s))} \\
& \leq & f(\x_{t-1}, \y^\star(\x_{t-1})) - f(\x_s, \y^\star(\x_{t-1})) + f(\x_s, \y^\star(\x_{t-1})) - f(\x_{t-1}, \y^\star(\x_s)) \nonumber \\
& \leq & f(\x_{t-1}, \y^\star(\x_{t-1})) - f(\x_s, \y^\star(\x_{t-1})) + f(\x_s, \y^\star(\x_s)) - f(\x_{t-1}, \y^\star(\x_s)). \nonumber
\end{eqnarray}
Since $f(\cdot, \y)$ is $L$-Lipschitz for any $\y \in \YCal$, we have
\begin{eqnarray*}
f(\x_{t-1}, \y^\star(\x_{t-1})) - f(\x_s, \y^\star(\x_{t-1})) & \leq & L\|\x_{t-1} - \x_s\| \ \leq \ \eta_\x L^2(t-1-s), \\
f(\x_s, \y^\star(\x_s)) - f(\x_{t-1}, \y^\star(\x_s)) & \leq & L\|\x_{t-1} - \x_s\| \ \leq \ \eta_\x L^2(t-1-s) \\
f(\x_{t-1}, \y_t) - f(\x_t, \y_t) & \leq & L\|\x_{t-1} - \x_t\| \ \leq \ \eta_\x L^2. 
\end{eqnarray*}
Putting these pieces together yields the first desired inequality. 

We proceed to the stochastic setting. For $\forall\y \in \YCal$, we use the similar argument to obtain
\begin{eqnarray*}
\lefteqn{\|\y - \y_t\|^2 \leq 2\eta_\y(\y_{t-1} - \y)^\top G_\y(\x_{t-1}, \y_{t-1}, \xi) + 2\eta_\y(\y_t - \y_{t-1})^\top \grady f(\x_{t-1}, \y_{t-1})} \\
& & + 2\eta_\y(\y_t - \y_{t-1})^\top\left(G_\y(\x_{t-1}, \y_{t-1}, \xi) - \grady f(\x_{t-1}, \y_{t-1})\right) + \|\y - \y_{t-1}\|^2 - \|\y_t - \y_{t-1}\|^2. 
\end{eqnarray*}
Using the Young's inequality, we have
\begin{equation*}
\eta_\y(\y_t - \y_{t-1})^\top (G_\y(\x_{t-1}, \y_{t-1}, \xi) - \grady f(\x_{t-1}, \y_{t-1})) \leq \tfrac{\|\y_t - \y_{t-1}\|^2}{4} + \eta_\y^2\|G_\y(\x_{t-1}, \y_{t-1}, \xi) - \grady f(\x_{t-1}, \y_{t-1})\|^2.  
\end{equation*}
Taking the expectation of both sides, conditioned on $(\x_{t-1}, \y_{t-1})$, together with Lemma~\ref{Lemma:SG-unbiased} yields
\begin{eqnarray*}
\lefteqn{\EE[\|\y - \y_t\|^2 \mid \x_{t-1}, \y_{t-1}] \leq 2\eta_\y(\y_{t-1} - \y)^\top \grady f(\x_{t-1}, \y_{t-1}) + 2\eta_\y \EE[(\y_t - \y_{t-1})^\top \grady f(\x_{t-1}, \y_{t-1}) \mid \x_{t-1}, \y_{t-1}]} \\
& & + 2\eta_\y^2 \EE[\|\grady f(\x_{t-1}, \y_{t-1}) - G_\y(\x_{t-1}, \y_{t-1}, \xi)\|^2 \mid \x_{t-1}, \y_{t-1}] + \|\y - \y_{t-1}\|^2 - \tfrac{\EE[\|\y_t - \y_{t-1}\|^2 \mid \x_{t-1}, \y_{t-1}]}{2}. 
\end{eqnarray*}
Taking the expectation of both sides together with Lemma~\ref{Lemma:SG-unbiased} yields that 
\begin{eqnarray*}
\EE[\|\y - \y_t\|^2] & \leq & 2\eta_\y\EE[(\y_{t-1} - \y)^\top \grady f(\x_{t-1}, \y_{t-1}) + (\y_t - \y_{t-1})^\top \grady f(\x_{t-1}, \y_{t-1})] \\
& & + \EE[\|\y - \y_{t-1}\|^2] - \tfrac{\EE[\|\y_t - \y_{t-1}\|^2]}{2} + \eta_\y^2\sigma^2. 
\end{eqnarray*}
Since $f(\x_{t-1}, \cdot)$ is concave and $\ell$-smooth, $\YCal$ is convex and $\eta_\y \leq 1/2\ell$, we have
\begin{equation*}
\EE[\|\y - \y_t\|^2] \leq \EE[\|\y - \y_{t-1}\|^2] + 2\eta_\y(f(\x_{t-1}, \y_t) - f(\x_{t-1}, \y)) + \eta_\y^2\sigma^2.  
\end{equation*}
Plugging $\y = \y^\star(\x_s)$ ($s \leq t-1$) in the above inequality yields that
\begin{equation*}
\EE[f(\x_{t-1}, \y^\star(\x_s)) - f(\x_{t-1}, \y_t)] \leq \tfrac{1}{2\eta_\y}\left(\EE[\|\y_{t-1} - \y^\star(\x_s)\|^2] - \EE[\|\y_t - \y^\star(\x_s)\|^2]\right) + \tfrac{\eta_\y \sigma^2}{2}. 
\end{equation*}
By the definition of $\Delta_{t-1}$, we have
\begin{eqnarray*}
\Delta_{t-1} & \leq & \EE[f(\x_{t-1}, \y^\star(\x_{t-1})) - f(\x_{t-1}, \y^\star(\x_s)) + (f(\x_t, \y_t) - f(\x_{t-1}, \y_{t-1})) + (f(\x_{t-1}, \y_t) - f(\x_t, \y_t))] \\
& & + \tfrac{\eta_\y\sigma^2}{2} + \tfrac{1}{2\eta_\y}\left(\EE[\|\y_{t-1} - \y^\star(\x_s)\|^2] - \EE[\|\y_t - \y^\star(\x_s)\|^2]\right). \nonumber 
\end{eqnarray*}
Using the fact that $f(\cdot, \y)$ is $L$-Lipschitz for $\forall\y \in \YCal$ and Lemma~\ref{Lemma:SG-unbiased}, we have
\begin{eqnarray*}
\EE[f(\x_{t-1}, \y^\star(\x_{t-1})) - f(\x_s, \y^\star(\x_{t-1}))] & \leq & \eta_\x L\sqrt{L^2 + \sigma^2} (t-1-s), \\
\EE[f(\x_s, \y^\star(\x_s)) - f(\x_{t-1}, \y^\star(\x_s))] & \leq & \eta_\x L\sqrt{L^2 + \sigma^2} (t-1-s), \\
\EE[f(\x_{t-1}, \y_t) - f(\x_t, \y_t)] & \leq & \eta_\x L\sqrt{L^2 + \sigma^2}. 
\end{eqnarray*}
Putting these pieces together with~\eqref{nc-neighborhood-first} yields the second desired inequality. 
\end{proof}
Without loss of generality, we assume that $B \leq T+1$ so that $(T+1)/B$ is an integer. The following lemma provides an upper bound for $\frac{1}{T+1}(\sum_{t=0}^T \Delta_t)$ using a localization technique. 
\begin{lemma}\label{Lemma:nc-key-objective}
For two-timescale GDA, let $\Delta_t = \Phi(\x_t) - f(\x_t, \y_t)$, the following statement holds true, 
\begin{equation*}
\tfrac{1}{T+1}\left(\sum_{t=0}^T \Delta_t\right) \leq \eta_\x L^2(B+1) + \tfrac{\ell D^2}{2B} + \tfrac{\widehat{\Delta}_0}{T+1}. 
\end{equation*}
For two-timescale SGDA, let $\Delta_t = \EE[\Phi(\x_t) - f(\x_t, \y_t)]$, the following statement holds true,
\begin{equation*}
\tfrac{1}{T+1}\left(\sum_{t=0}^T \Delta_t\right) \leq \eta_\x L\sqrt{L^2+\sigma^2}(B+1) + \tfrac{D^2}{2B\eta_\y} + \tfrac{\eta_\y\sigma^2}{2} + \tfrac{\widehat{\Delta}_0}{T+1}. 
\end{equation*}
\end{lemma}
\begin{proof}
We consider the deterministic setting. In particular, we divide $\{\Delta_t\}_{t=0}^T$ into several blocks where each block contains at most $B$ terms, given by
\begin{equation*}
\{\Delta_t\}_{t=0}^{B-1}, \{\Delta_t\}_{t=B}^{2B-1}, \ldots, \{\Delta_t\}_{T-B+1}^T. 
\end{equation*}
Then we have
\begin{equation}\label{nc-objective-first}
\tfrac{1}{T+1}\left(\sum_{t=0}^T \Delta_t\right) \leq \tfrac{B}{T+1}\left[\sum_{j=0}^{(T+1)/B-1} \left(\tfrac{1}{B}\sum_{t= jB}^{(j+1)B-1} \Delta_t\right)\right]. 
\end{equation}
Letting $s = 0$ in the first inequality in Lemma~\eqref{Lemma:nc-key-neighborhood} yields that
\begin{eqnarray}\label{nc-objective-second}
\sum_{t = 0}^{B-1} \Delta_t & \leq & \eta_\x L^2B^2 + \tfrac{\ell}{2}\|\y_0 - \y^\star(\x_0)\|^2 + (f(\x_B, \y_B) - f(\x_0, \y_0)) \\ 
& \leq & \eta_\x L^2 B^2 + \tfrac{\ell D^2}{2} + (f(\x_B, \y_B) - f(\x_0, \y_0)). \nonumber
\end{eqnarray}
Similarly, letting $s = jB$ yields that, for $1 \leq j \leq \frac{T+1}{B} - 1$, 
\begin{equation}\label{nc-objective-third}
\sum_{t= jB}^{(j+1)B-1} \Delta_t \leq \eta_\x L^2 B^2 + \tfrac{\ell D^2}{2} + (f(\x_{jB+B}, \y_{jB+B}) - f(\x_{jB}, \y_{jB})). 
\end{equation}
Plugging~\eqref{nc-objective-second} and~\eqref{nc-objective-third} into~\eqref{nc-objective-first} yields 
\begin{equation}\label{nc-objective-fourth}
\tfrac{1}{T+1}\left(\sum_{t=0}^T \Delta_t\right) \leq \eta_\x L^2B + \tfrac{\ell D^2}{2B} + \tfrac{f(\x_{T+1}, \y_{T+1}) - f(\x_0, \y_0)}{T+1}. 
\end{equation}
Since $f(\cdot, \y)$ is $L$-Lipschitz for $\forall \y \in \YCal$, we have
\begin{eqnarray}\label{nc-objective-fifth}
f(\x_{T+1}, \y_{T+1}) - f(\x_0, \y_0) & = & f(\x_{T+1}, \y_{T+1}) - f(\x_0, \y_{T+1}) + f(\x_0, \y_{T+1}) - f(\x_0, \y_0) \nonumber \\
& \leq & \eta_\x L^2(T+1) + \widehat{\Delta}_0. 
\end{eqnarray}
Plugging~\eqref{nc-objective-fifth} into~\eqref{nc-objective-fourth} yields the first desired inequality. As for the stochastic setting, letting $s = jB$ in the second inequality in Lemma~\ref{Lemma:nc-key-neighborhood} yields that
\begin{equation}\label{nc-objective-sixth}
\sum_{t= jB}^{(j+1)B-1} \Delta_t \leq \eta_\x L\sqrt{L^2 + \sigma^2} B^2 + \tfrac{D^2}{2\eta_\y} + \tfrac{\eta_\y \sigma^2}{2}, \quad 0 \leq j \leq \tfrac{T+1}{B} - 1. 
\end{equation}
Using the similar argument with~\eqref{nc-objective-first} and~\eqref{nc-objective-sixth} yields the second desired inequality.
\end{proof}

\subsection{Proof of Theorem~\ref{Theorem:nc-GDA-app}}
Summing up the first inequality in Lemma~\ref{Lemma:nc-key-descent} over $t = 1, 2, \ldots, T+1$ yields that 
\begin{equation*}
\Phi_{1/2\ell}(\x_{T+1}) \leq \Phi_{1/2\ell}(\x_0) + 2\eta_\x\ell\left(\sum_{t=0}^T \Delta_t\right) - \tfrac{\eta_\x}{4} \left(\sum_{t=0}^T \|\grad \Phi_{1/2\ell}(\x_t)\|^2\right) + \eta_\x^2 \ell L^2 (T+1). 
\end{equation*}
Combining the above inequality with the first inequality in Lemma~\ref{Lemma:nc-key-objective} yields that
\begin{equation*}
\Phi_{1/2\ell}(\x_{T+1}) \leq \Phi_{1/2\ell}(\x_0) + 2\eta_\x\ell(T+1)\left(\eta_\x L^2(B+1) + \tfrac{\ell D^2}{2B}\right) + 2\eta_\x\ell\widehat{\Delta}_0 - \tfrac{\eta_\x}{4}\left(\sum_{t=0}^T \|\grad \Phi_{1/2\ell}(\x_t)\|^2\right) + \eta_\x^2 \ell L^2 (T+1). 
\end{equation*}
By the definition of $\widehat{\Delta}_\Phi$, we have
\begin{equation*}
\tfrac{1}{T+1}\left(\sum_{t=0}^T \|\grad \Phi_{1/2\ell}(\x_t)\|^2\right) \leq \tfrac{4\widehat{\Delta}_\Phi}{\eta_\x (T+1)} + 8\ell\left(\eta_\x(B+1)L^2 + \tfrac{\ell D^2}{2B}\right) + \tfrac{8\ell\widehat{\Delta}_0}{T+1} + 4\eta_\x\ell L^2. 
\end{equation*}
Letting $B=1$ for $D=0$ and $B = \frac{D}{2L}\sqrt{\frac{\ell}{\eta_\x}}$ for $D>0$, we have
\begin{equation*}
\tfrac{1}{T+1}\left(\sum_{t=0}^T \|\grad \Phi_{1/2\ell}(\x_t)\|^2\right) \leq \tfrac{4\widehat{\Delta}_\Phi}{\eta_\x (T+1)} + \tfrac{8\ell\widehat{\Delta}_0}{T+1} + 16\ell LD\sqrt{\ell\eta_\x} + 4\eta_\x\ell L^2. 
\end{equation*}
Since $\eta_\x = \min\left\{\frac{\epsilon^2}{16\ell L^2}, \ \frac{\epsilon^4}{4096\ell^3 L^2 D^2}\right\}$, we have
\begin{equation*}
\tfrac{1}{T+1}\left(\sum_{t=0}^T \|\grad \Phi_{1/2\ell}(\x_t)\|^2\right) \leq \tfrac{4\widehat{\Delta}_\Phi}{\eta_\x(T+1)} + \tfrac{8\ell\widehat{\Delta}_0}{T+1} + \tfrac{\epsilon^2}{2}. 
\end{equation*}
This implies that the number of iterations required by Algorithm \ref{Algorithm:GDA} to return an $\epsilon$-stationary point is bounded by 
\begin{equation*}
O\left(\left(\frac{\ell L^2\widehat{\Delta}_\Phi}{\epsilon^4} + \frac{\ell\widehat{\Delta}_0}{\epsilon^2}\right)\max\left\{1, \ \frac{\ell^2D^2}{\epsilon^2}\right\}\right),  
\end{equation*}
which gives the same total gradient complexity. 

\subsection{Proof of Theorem~\ref{Theorem:nc-SGDA-app}}
Summing up the second inequality in Lemma~\ref{Lemma:nc-key-descent} over $t = 1, 2, \ldots, T+1$ yields that 
\begin{equation*}
\EE[\Phi_{1/2\ell}(\x_{T+1})] \leq \Phi_{1/2\ell}(\x_0) + 2\eta_\x\ell\left(\sum_{t=0}^T \Delta_t\right) - \tfrac{\eta_\x}{4} \left(\sum_{t=0}^T \EE[\|\grad \Phi_{1/2\ell}(\x_t)\|^2]\right) + \eta_\x^2 \ell (L^2 + \sigma^2)(T+1). 
\end{equation*}
Combining the above inequality with the second inequality in Lemma~\ref{Lemma:nc-key-objective} yields that
\begin{eqnarray*}
\EE[\Phi_{1/2\ell}(\x_{T+1})] & \leq & \Phi_{1/2\ell}(\x_0) + 2\eta_\x\ell(T+1)\left(\eta_\x L\sqrt{L^2+\sigma^2} (B+1) + \tfrac{D^2}{2B\eta_\y} + \tfrac{\eta_\y \sigma^2}{2}\right) + 2\eta_\x\ell\widehat{\Delta}_0 \\ 
& & - \tfrac{\eta_\x}{4} \left(\sum_{t=0}^T \EE[\|\grad \Phi_{1/2\ell}(\x_t)\|^2]\right) + \eta_\x^2 \ell (L^2 + \sigma^2)(T+1). 
\end{eqnarray*}
By the definition of $\widehat{\Delta}_\Phi$, we have
\begin{equation*}
\tfrac{1}{T+1}\left(\sum_{t=0}^T \EE[\|\grad \Phi_{1/2\ell}(\x_t)\|^2]\right) \leq \tfrac{4\widehat{\Delta}_\Phi}{\eta_\x (T+1)} + 8\ell\left(\eta_\x L\sqrt{L^2+\sigma^2} (B+1) + \tfrac{D^2}{2B\eta_\y} + \tfrac{\eta_\y \sigma^2}{2}\right) + \tfrac{8\ell\widehat{\Delta}_0}{T+1} + 4\eta_\x\ell (L^2 + \sigma^2). 
\end{equation*}
Letting $B=1$ for $D=0$ and $B = \frac{D}{2}\sqrt{\frac{1}{\eta_\x\eta_\y L\sqrt{L^2+\sigma^2}}}$ for $D>0$, we have
\begin{equation*}
\tfrac{1}{T+1}\left(\sum_{t=0}^T \|\grad \Phi_{1/2\ell}(\x_t)\|^2\right) \leq \tfrac{4\widehat{\Delta}_\Phi}{\eta_\x (T+1)} + \tfrac{8\ell\widehat{\Delta}_0}{T+1} + 16\ell D\sqrt{\tfrac{\eta_\x L\sqrt{L^2+\sigma^2}}{\eta_\y}} + 4\eta_\y\ell\sigma^2 + 4\eta_\x\ell (L^2 + \sigma^2). 
\end{equation*}
Since $\eta_\x = \min\left\{\frac{\epsilon^2}{16\ell\left(L^2 + \sigma^2\right)}, \ \frac{\epsilon^4}{8192\ell^3 D^2 L\sqrt{L^2+\sigma^2}}, \ \frac{\epsilon^6}{65536\ell^3 D^2\sigma^2 L\sqrt{L^2+\sigma^2}}\right\}$ and $\eta_\y = \min\left\{\frac{1}{2\ell}, \ \frac{\epsilon^2}{16\ell\sigma^2}\right\}$, we have
\begin{equation*}
\tfrac{1}{T+1}\left(\sum_{t=0}^T \|\grad \Phi_{1/2\ell}(\x_t)\|^2\right) \leq \tfrac{4\widehat{\Delta}_\Phi}{\eta_\x (T+1)} + \tfrac{8\ell\widehat{\Delta}_0}{T+1} + \tfrac{3\epsilon^2}{4}. 
\end{equation*}
This implies that the number of iterations required by Algorithm \ref{Algorithm:SGDA} to return an $\epsilon$-stationary point is bounded by 
\begin{equation*}
O\left(\left(\frac{\ell\left(L^2 + \sigma^2\right)\widehat{\Delta}_\Phi}{\epsilon^4} + \frac{\ell\widehat{\Delta}_0}{\epsilon^2}\right)\max\left\{1, \ \frac{\ell^2D^2}{\epsilon^2}, \ \frac{\ell^2 D^2\sigma^2}{\epsilon^4}\right\}\right),  
\end{equation*}
which gives the same total gradient complexity. 

\section{Results for GDmax and SGDmax}\label{sec:results_max}
For the sake of completeness, we present GDmax and SGDmax in Algorithm~\ref{Algorithm:GDmax} and~\ref{Algorithm:SGDmax}. For any given $\x_t \in \br^m$, the max-oracle approximately solves $\max_{\y \in \YCal} f(\x_t, \y)$ at each iteration. Although GDmax and SGDmax are easier to understand, they have two disadvantages over two-timescale GDA and SGDA: 1) Both GDmax and SGDmax are nested-loop algorithms. Since it is difficult to pre-determine the number iterations for the inner loop, these algorithms are not favorable in practice; 2) In the general setting where $f(\x, \cdot)$ is nonconcave, GDmax and SGDmax are inapplicable as we can not efficiently solve the maximization problem to a global optimum. Nevertheless, we present the complexity bound for GDmax and SGDmax for the sake of completeness. It is worth noting that a portion of results were derived before~\citet{Jin-2019-Minmax} and~\citet{Nouiehed-2019-Solving} and our proof depends on the same techniques.

For nonconvex-strongly-convex problems, the target is to find an $\epsilon$-stationary point (cf. Definition~\ref{def:nsc-stationary}) given only gradient (or stochastic gradient) access to $f$. Denote $\Delta_\Phi = \Phi(\x_0) - \min_{\x \in \br^m} \Phi(\x)$, we present the gradient complexity for GDmax in the following theorem. 
\begin{theorem}\label{Theorem:nsc-GDmax-complexity-bound}
Under Assumption~\ref{Assumption:nsc} and letting the step size $\eta_\x > 0$ and the tolerance for the max-oracle $\zeta > 0$ be $\eta_\x = 1/[8\kappa\ell]$ and $\zeta = \epsilon^2/[6\ell]$, the number of iterations required by Algorithm \ref{Algorithm:GDmax} to return an $\epsilon$-stationary point is bounded by $O(\kappa\ell\Delta_\Phi\epsilon^{-2})$. Furthermore, the $\zeta$-accurate max-oracle can be realized by gradient ascent (GA) with the stepsize $\eta_\y = 1/\ell$ for $O(\kappa\log(\ell D^2/\zeta))$ iterations, which gives the total gradient complexity of the algorithm: 
\begin{equation*}
O\left(\frac{\kappa^2\ell\Delta_\Phi}{\epsilon^2}\log\left(\frac{\ell D}{\epsilon}\right)\right). 
\end{equation*}
\end{theorem}
Theorem~\ref{Theorem:nsc-GDmax-complexity-bound} shows that, if we alternate between one-step gradient descent over $\x$ and $O(\kappa\log(\ell D/\epsilon))$ gradient ascent steps over $\y$ with a pair of proper learning rates $(\eta_\x, \eta_\y)$, we find at least one stationary point of $\Phi$ within $O(\kappa^2\epsilon^{-2}\log(\ell/\epsilon))$ gradient evaluations. Then we present similar guarantees when only stochastic gradients are available in the following theorem. 
\begin{theorem}\label{Theorem:nsc-SGDmax-complexity-bound}
Under Assumption~\ref{Assumption:stoc-oracle} and~\ref{Assumption:nsc} and letting the step size $\eta_\x > 0$ and the tolerance for the max-oracle $\zeta > 0$ be the same in Theorem~\ref{Theorem:nsc-GDmax-complexity-bound} with the batch size $M = \max\{1, 12\kappa\sigma^2\epsilon^{-2}\}$, 
the number of iterations required by Algorithm \ref{Algorithm:SGDmax} to return an $\epsilon$-stationary point is bounded by $O(\kappa\ell\Delta_\Phi\epsilon^{-2})$. Furthermore, the $\zeta$-accurate max-oracle can be realized by mini-batch stochastic gradient ascent (SGA) with the step size $\eta_\y = 1/\ell$ and the mini-batch size $M = \max\{1, 2\sigma^2\kappa\ell^{-1}\zeta^{-1}\}$ for $O(\kappa\log(\ell D^2/\zeta)\max\{1, 2\sigma^2\kappa\ell^{-1}\zeta^{-1}\})$ gradient evaluations, which gives the total gradient complexity of the algorithm: 
\begin{equation*}
O\left(\frac{\kappa^2\ell\Delta_\Phi}{\epsilon^2}\log\left(\frac{\ell D}{\epsilon}\right)\max\left\{1, \ \frac{\kappa\sigma^2}{\epsilon^2}\right\}\right). 
\end{equation*}
\end{theorem}
The sample size $M = O(\kappa\sigma^2\epsilon^{-2})$ guarantees that the variance is less than $\epsilon^2/\kappa$ so that the average stochastic gradients over the batch are sufficiently close to the true gradients $\gradx f$ and $\grady f$. 

We now proceed to the theoretical guarantee for GDmax and SGDmax algorithms for nonconvex-concave problems. The target is to find an $\epsilon$-stationary point of a weakly convex function (Definition~\ref{def:nc-stationary}) given only gradient (or stochastic gradient) access to $f$. Denote $\widehat{\Delta}_\Phi = \Phi_{1/2\ell}(\x_0) - \min_{\x \in \br^m} \Phi_{1/2\ell}(\x)$, we present the gradient complexity for GDmax and SGDmax in the following two theorems.
\begin{theorem}\label{Theorem:nc-GDmax-complexity-bound}
Under Assumption~\ref{Assumption:nc} and letting the step size $\eta_\x > 0$ and the tolerance for the max-oracle $\zeta > 0$ be $\eta_\x = \epsilon^2/[\ell L^2]$ and $\zeta = \epsilon^2/[24\ell]$, the number of iterations required by Algorithm \ref{Algorithm:GDmax} to return an $\epsilon$-stationary point is bounded by $O(\ell L^2\widehat{\Delta}_\Phi\epsilon^{-4})$. Furthermore, the $\zeta$-accurate max-oracle is realized by GA with the step size $\eta_\y = 1/2\ell$ for $O(\ell D^2/\zeta)$ iterations, which gives the total gradient complexity of the algorithm: 
\begin{equation*}
O\left(\frac{\ell^3 L^2 D^2 \widehat{\Delta}_\Phi}{\epsilon^6}\right). 
\end{equation*}
\end{theorem}
\begin{theorem}\label{Theorem:nc-SGDmax-complexity-bound}
Under Assumptions~\ref{Assumption:stoc-oracle} and~\ref{Assumption:nc} and letting the tolerance for the max-oracle $\zeta > 0$ be chosen as the same as in Theorem~\ref{Theorem:nc-GDmax-complexity-bound} with a step size $\eta_\x > 0$ and a batch size $M > 0$ given by $\eta_\x = \epsilon^2/[\ell(L^2 + \sigma^2)]$ and $M = 1$, the number of iterations required by Algorithm \ref{Algorithm:SGDmax} to return an $\epsilon$-stationary point is bounded by $O(\ell(L^2 + \sigma^2)\widehat{\Delta}_\Phi\epsilon^{-4})$. Furthermore, the $\zeta$-accurate max-oracle is realized by SGA with the step size $\eta_\y = \min\{1/2\ell, \epsilon^2/[\ell\sigma^2]\}$ and a batch size $M = 1$ for $O(\ell D^2\zeta^{-1}\max\{1, \sigma^2\ell^{-1}\zeta^{-1}\})$ iterations, which gives the following total gradient complexity of the algorithm: 
\begin{equation*}
O\left(\frac{\ell^3(L^2+\sigma^2)D^2 \widehat{\Delta}_\Phi}{\epsilon^6}\max\left\{1, \ \frac{\sigma^2}{\epsilon^2}\right\}\right). 
\end{equation*}
\end{theorem}
When $\sigma^2 \lesssim \varepsilon^2$, the stochastic gradients are sufficiently close to the true gradients $\gradx f$ and $\grady f$ and the gradient complexity of SGDmax matches that of GDmax.
\begin{algorithm}[!t]
\caption{Gradient Descent with Max-oracle (GDmax)}\label{Algorithm:GDmax}
\begin{algorithmic}
\renewcommand{\algorithmicrequire}{\textbf{Input: }}
\renewcommand{\algorithmicensure}{\textbf{Output: }}
\REQUIRE initial point $\x_0$, learning rate $\eta_\x$ and max-oracle accuracy $\zeta$.
\FOR{$t = 1, 2, \ldots$}
\STATE find $\y_{t-1} \in \YCal$ so that $f(\x_{t-1}, \y_{t-1}) \geq \max_{\y \in \YCal} f(\x_{t-1}, \y) - \zeta$.
\STATE $\x_t \leftarrow \x_{t-1} - \eta_\x \grad_\x f(\x_{t-1}, \y_{t-1})$.
\ENDFOR
\end{algorithmic}
\end{algorithm}
\begin{algorithm}[!t]
\caption{Stochastic Gradient Descent with Max-oracle (SGDmax)}\label{Algorithm:SGDmax}
\begin{algorithmic}
\renewcommand{\algorithmicrequire}{\textbf{Input: }}
\renewcommand{\algorithmicensure}{\textbf{Output: }}
\REQUIRE initial point $\x_0$, learning rate $\eta_\x$ and max-oracle accuracy $\zeta$.
\FOR{$t = 1, 2, \ldots$}
\STATE Draw a collection of i.i.d. data samples $\{\xi_i\}_{i=1}^M$. 
\STATE find $\y_{t-1} \in \YCal$ so that $\E[f(\x_{t-1}, \y_{t-1}) \mid \x_{t-1}] \geq \max_{\y \in \YCal} f(\x_{t-1}, \y) - \zeta$.
\STATE $\x_t \leftarrow \x_{t-1} - \eta_\x\left(\frac{1}{M} \sum_{i=1}^M G_\x\left(\x_{t-1}, \y_{t-1}, \xi_i\right)\right)$. 
\ENDFOR
\end{algorithmic}
\end{algorithm}
\subsection{Proof of Theorem~\ref{Theorem:nsc-GDmax-complexity-bound}}
We present the gradient complexity bound of the gradient-ascent-based $\zeta$-accurate max-oracle in the following lemma. 
\begin{lemma}\label{Lemma:nsc-GDmax-neighborhood}
Let $\zeta > 0$ be given, the $\zeta$-accurate max-oracle can be realized by running gradient ascent with a step size $\eta_\y = 1/\ell$ for 
\begin{equation*}
O\left(\kappa \log\left(\frac{\ell D^2}{\zeta}\right)\right)
\end{equation*}
gradient evaluations. In addition, the output $\y$ satisfies $\|\y^\star - \y\|^2 \leq \zeta/\ell$, where $\y^\star$ is the exact maximizer. 
\end{lemma}
\begin{proof}
Since $f(\x_t, \cdot)$ is $\mu$-strongly concave, we have
\begin{equation*}
f(\x_t, \y^\star(\x_t)) - f(\x_t, \y_t) \leq \left(1 - \frac{1}{\kappa}\right)^{N_t}\tfrac{\ell D^2}{2},  \qquad \|\y^\star(\x_t) - \y_t\|^2 \leq \left(1 - \tfrac{1}{\kappa}\right)^{N_t}D^2. 
\end{equation*}
The first inequality implies that the number of iterations required is $O(\kappa \log(\ell D^2/\zeta))$ which is also the number of gradient evaluations. This together with the second inequality yields the other results. 
\end{proof}
\paragraph{Proof of Theorem~\ref{Theorem:nsc-GDmax-complexity-bound}:} It is easy to find that the first descent inequality in Lemma~\ref{Lemma:nsc-key-descent} is applicable to GDmax: 
\begin{equation}\label{nsc-GDmax-first}
\Phi(\x_t) \leq \Phi(\x_{t-1}) - \left(\tfrac{\eta_{\x}}{2} - 2\eta_{\x}^2\kappa\ell\right) \|\grad \Phi(\x_{t-1})\|^2 + \left(\tfrac{\eta_{\x}}{2} + 2\eta_{\x}^2\kappa\ell\right)\|\grad\Phi(\x_{t-1}) - \gradx f(\x_{t-1}, \y_{t-1})\|^2.
\end{equation}
Since $\grad \Phi(\x_{t-1}) = \gradx f\left(\x_{t-1}, \y^\star(\x_{t-1})\right)$, we have
\begin{equation}\label{nsc-GDmax-second}
\|\grad\Phi(\x_{t-1}) - \gradx f(\x_{t-1}, \y_{t-1})\|^2 \leq \ell^2\|\y^\star(\x_{t-1}) - \y_{t-1}\|^2 \leq \ell\zeta.  
\end{equation}
Since $\eta_\x = 1/8\kappa\ell$, we have
\begin{equation}\label{nsc-GDmax-stepsize}
\tfrac{\eta_\x}{4} \leq \tfrac{\eta_\x}{2} - 2\eta_\x^2\kappa\ell \leq \tfrac{\eta_\x}{2} + 2\eta_\x^2\kappa\ell \leq \tfrac{3\eta_\x}{4}.  
\end{equation}
Plugging~\eqref{nsc-GDmax-second} and~\eqref{nsc-GDmax-stepsize} into~\eqref{nsc-GDmax-first} yields that
\begin{equation}\label{nsc-GDmax-third}
\Phi(\x_t) \leq \Phi(\x_{t-1}) - \tfrac{\eta_\x}{4} \|\grad \Phi(\x_{t-1})\|^2 + \tfrac{3\eta_\x\ell\zeta}{4}. 
\end{equation}
Summing up~\eqref{nsc-GDmax-third} over $t = 1, 2, \ldots, T+1$ and rearranging the terms yields that
\begin{equation*}
\tfrac{1}{T+1}\left(\sum_{t=0}^T \|\grad \Phi(\x_t)\|^2\right) \leq \tfrac{4(\Phi(\x_0) - \Phi(\x_{T+1}))}{\eta_\x(T+1)} + 3\ell\zeta. 
\end{equation*}
By the definition of $\eta_\x$ and $\Delta_\Phi$, we conclude that 
\begin{equation*}
\tfrac{1}{T+1}\left(\sum_{t=0}^T \|\grad \Phi(\x_t)\|^2\right) \leq \tfrac{32\kappa\ell\Delta_\Phi}{T+1} + 3\ell\zeta. 
\end{equation*}
This implies that the number of iterations required by Algorithm \ref{Algorithm:GDmax} to return an $\epsilon$-stationary point is bounded by 
\begin{equation*}
O\left(\frac{\kappa\ell\Delta_\Phi}{\epsilon^2}\right). 
\end{equation*}
Combining Lemma~\ref{Lemma:nsc-GDmax-neighborhood} gives the total gradient complexity of Algorithm \ref{Algorithm:GDmax}:
\begin{equation*}
O\left(\frac{\kappa^2\ell\Delta_\Phi}{\epsilon^2}\log\left(\frac{\ell D}{\epsilon}\right)\right). 
\end{equation*}
This completes the proof. 

\subsection{Proof of Theorem~\ref{Theorem:nsc-SGDmax-complexity-bound}}
We present the gradient complexity bound of the stochastic-gradient-ascent-based $\zeta$-accurate max-oracle in terms of stochastic gradient in the following lemma. 
\begin{lemma}\label{Lemma:nsc-SGDmax-neighborhood}
Let $\zeta > 0$ be given, the $\zeta$-accurate max-oracle can be realized by running stochastic gradient ascent with a step size $\eta_\y = 1/\ell$ and a batch size $M = \max\{1, 2\sigma^2\kappa/\ell\zeta\}$ for 
\begin{equation*}
O\left(\kappa\log\left(\frac{\ell D^2}{\zeta}\right)\max\left\{1, \ \frac{2\sigma^2\kappa}{\ell\zeta}\right\}\right)  
\end{equation*}
stochastic gradient evaluations. In addition, the output $\y$ satisfies $\|\y^\star - \y\|^2 \leq \zeta/\ell$ where $\y^\star$ is the exact maximizer. 
\end{lemma}
\begin{proof}
Since $f(\x_t, \cdot)$ is $\mu$-strongly concave, we have
\begin{equation*}
\EE[f(\x_t, \y^\star(\x_t)) - f(\x_t, \y_t)] \leq \left(1 - \tfrac{1}{\kappa}\right)^{N_t} \tfrac{\ell D^2}{2} + \tfrac{\eta_\y^2\ell\sigma^2}{M}\left(\sum_{j=0}^{N_{t-1}} (1 - \mu\eta_\y)^{N_{t-1}-1-j}\right) \leq \left(1 - \tfrac{1}{\kappa}\right)^{N_t}\tfrac{\ell D^2}{2} + \tfrac{\sigma^2\kappa}{\ell M},  
\end{equation*}
and 
\begin{equation*}
\EE[\|\y^\star(\x_t)) - \y_t\|^2] \leq \left(1 - \tfrac{1}{\kappa}\right)^{N_t} D^2 + \tfrac{\eta_\y^2\sigma^2}{M}\left(\sum_{j=0}^{N_{t-1}} (1 - \mu\eta_\y)^{N_{t-1}-1-j}\right) \leq \left(1 - \tfrac{1}{\kappa}\right)^{N_t}\tfrac{\ell D^2}{2} + \tfrac{\sigma^2\kappa}{\ell^2 M}.   
\end{equation*}
The first inequality implies that the number of iterations is $O(\kappa\log(\ell D^2/\zeta))$ and the number of stochastic gradient evaluation is $O(\kappa\log(\ell D^2/\zeta)\max\{1, 2\sigma^2\kappa/\ell\zeta\})$. This together with the second inequality yields the other results.  
\end{proof}
\paragraph{Proof of Theorem~\ref{Theorem:nsc-SGDmax-complexity-bound}:} It is easy to find that the second descent inequality in Lemma~\ref{Lemma:nsc-key-descent} is applicable to SGDmax: 
\begin{eqnarray}\label{nsc-SGDmax-first}
\lefteqn{\EE[\Phi(\x_t)] \leq \EE[\Phi(\x_{t-1})] - \left(\tfrac{\eta_{\x}}{2} - 2\eta_{\x}^2\kappa\ell\right) \EE[\|\grad \Phi(\x_{t-1})\|^2]} \\
& & + \left(\tfrac{\eta_{\x}}{2} + 2\eta_{\x}^2\kappa\ell\right)\EE[\|\grad\Phi(\x_{t-1}) - \gradx f(\x_{t-1}, \y_{t-1})\|^2] + \tfrac{\eta_\x^2\kappa\ell\sigma^2}{M}. \nonumber
\end{eqnarray}
Since $\grad \Phi(\x_{t-1}) = \gradx f\left(\x_{t-1}, \y^\star(\x_{t-1})\right)$, we have
\begin{equation}\label{nsc-SGDmax-second}
\EE[\|\grad\Phi(\x_t) - \gradx f(\x_t, \y_t)\|^2] \leq \ell^2\EE[\|\y^\star(\x_t) - \y_t\|^2] \leq \ell\zeta. 
\end{equation}
Since $\eta_\x = 1/8\kappa\ell$, we have~\eqref{nsc-GDmax-stepsize}. Plugging~\eqref{nsc-GDmax-stepsize} and~\eqref{nsc-SGDmax-second} into~\eqref{nsc-SGDmax-first} yields that
\begin{equation}\label{nsc-SGDmax-third}
\EE[\Phi(\x_t)] \leq \EE[\Phi(\x_{t-1})] - \tfrac{\eta_\x}{4} \EE[\|\grad \Phi(\x_{t-1})\|^2] + \tfrac{3\eta_\x\ell\zeta}{4} + \tfrac{\eta_\x^2\kappa\ell\sigma^2}{M}. 
\end{equation}
Summing up~\eqref{nsc-SGDmax-third} over $t = 1, 2, \ldots, T+1$ and rearranging the terms yields that
\begin{equation*}
\tfrac{1}{T+1}\left(\sum_{t=0}^T \EE[\|\grad \Phi(\x_t)\|^2]\right) \leq \tfrac{4(\Phi(\x_0) - \EE[\Phi(\x_{T+1})])}{\eta_\x(T+1)} + 3\ell\zeta + \tfrac{4\eta_\x\kappa\ell\sigma^2}{M}. 
\end{equation*}
By the definition of $\eta_\x$ and $\Delta_\Phi$, we conclude that 
\begin{equation*}
\tfrac{1}{T+1}\left(\sum_{t=0}^T \EE[\|\grad \Phi(\x_t)\|^2]\right) \leq \tfrac{32\kappa\ell\Delta_\Phi}{T+1} + 3\ell\zeta + \tfrac{\sigma^2}{2M}. 
\end{equation*}
This implies that the number of iterations required by Algorithm \ref{Algorithm:SGDmax} to return an $\epsilon$-stationary point is bounded by 
\begin{equation*}
O\left(\frac{\kappa\ell\Delta_\Phi}{\epsilon^2}\right).   
\end{equation*}
Note that the same batch set can be reused to construct the unbiased stochastic gradients for both $\gradx f(\x_{t-1}, \y_{t-1})$ and $\grady f(\x_{t-1}, \y_{t-1})$ at each iteration. Combining Lemma~\ref{Lemma:nsc-SGDmax-neighborhood} gives the total gradient complexity of Algorithm \ref{Algorithm:SGDmax}:
\begin{equation*}
O\left(\frac{\kappa^2\ell\Delta_\Phi}{\epsilon^2}\log\left(\frac{\sqrt{\kappa}\ell D}{\epsilon}\right)\max\left\{1, \ \frac{\sigma^2\kappa^2}{\epsilon^2}\right\}\right). 
\end{equation*} 
This completes the proof. 

\subsection{Proof of Theorem~\ref{Theorem:nc-GDmax-complexity-bound}}
We present the gradient complexity bound of the gradient-ascent-based $\zeta$-accurate max-oracle in the following lemma. 
\begin{lemma}\label{Lemma:nc-GDmax-neighborhood}
Let $\zeta > 0$ be given, the $\zeta$-accurate max-oracle can be realized by running gradient ascent with a step size $\eta_\y = 1/2\ell$ for 
\begin{equation*}
O\left(\max\left\{1, \ \frac{2\ell D^2}{\zeta}\right\}\right)
\end{equation*}
gradient evaluations. 
\end{lemma}
\begin{proof}
Since $f(\x_t, \cdot)$ is concave, we have $f(\x_t, \y^\star(\x_t)) - f(\x_t, \y_t) \leq \frac{2\ell D^2}{N_t}$ which implies that the number of iterations required is $\OCal\left(\max\left\{1, \ \frac{2\ell D^2}{\zeta}\right\}\right)$ which is the number of gradient evaluation. 
\end{proof}
\paragraph{Proof of Theorem~\ref{Theorem:nc-GDmax-complexity-bound}:} It is easy to find that the first descent inequality in Lemma~\ref{Lemma:nc-key-descent} is applicable to GDmax: 
\begin{equation}\label{nc-GDmax-first}
\Phi_{1/2\ell}(\x_t) \leq \Phi_{1/2\ell}(\x_{t-1}) + 2\eta_\x\ell\Delta_{t-1} - \tfrac{\eta_\x}{4}\|\grad \Phi_{1/2\ell}(\x_{t-1})\|^2 + \eta_\x^2 \ell L^2. 
\end{equation}
Summing up~\eqref{nc-GDmax-first} over $T = 1, 2, \ldots, T+1$ together with $\Delta_{t-1} \leq \zeta$ and rearranging the terms yields
\begin{equation*}
\tfrac{1}{T+1}\left(\sum_{t=0}^T \|\grad \Phi_{1/2\ell}(\x_t)\|^2\right) \leq \tfrac{4(\Phi_{1/2\ell}(\x_0) - \Phi_{1/2\ell}(\x_{T+1}))}{\eta_\x(T+1)} + 8\ell\zeta + 4\eta_\x\ell L^2. 
\end{equation*}
By the definition of $\eta_\x$ and $\widehat{\Delta}_\Phi$, we have
\begin{equation*}
\tfrac{1}{T+1} \left(\sum_{t=0}^T \|\grad \Phi_{1/2\ell}(\x_t)\|^2\right) \leq \tfrac{48\ell L^2\widehat{\Delta}_\Phi}{\epsilon^2(T+1)} + 8\ell\zeta + \tfrac{\epsilon^2}{3}. 
\end{equation*}
This implies that the number of iterations required by Algorithm \ref{Algorithm:GDmax} to return an $\epsilon$-stationary point is bounded by 
\begin{equation*}
O\left(\frac{\ell L^2\widehat{\Delta}_\Phi}{\epsilon^4}\right). 
\end{equation*}
Combining Lemma~\ref{Lemma:nc-GDmax-neighborhood} gives the total gradient complexity of Algorithm \ref{Algorithm:GDmax}:
\begin{equation*}
O\left(\frac{\ell L^2\widehat{\Delta}_\Phi}{\epsilon^4}\max\left\{1, \ \frac{\ell^2D^2}{\epsilon^2}\right\}\right). 
\end{equation*}
This completes the proof.

\subsection{Proof of Theorem~\ref{Theorem:nc-SGDmax-complexity-bound}}
We present the gradient complexity bound of the stochastic-ascent-based $\zeta$-accurate max-oracle in the following lemma.
\begin{lemma}\label{Lemma:nc-SGDmax-neighborhood}
Let $\zeta > 0$ be given, the $\zeta$-accurate max-oracle can be realized by running stochastic gradient ascent with a step size $\eta_\y = \min\{1/2\ell, \zeta/2\sigma^2\}$ and a batch size $M=1$ for 
\begin{equation}
O\left(\max\left\{1, \ \frac{4\ell D^2}{\zeta}, \ \frac{4\sigma^2 D^2}{\zeta^2}\right\}\right)
\end{equation}
stochastic gradient evaluations. 
\end{lemma}
\begin{proof}
Since $f(\x_t, \cdot)$ is concave and $\eta_\y = \min\{\frac{1}{2\ell}, \frac{\zeta}{2\sigma^2}\}$, we have $\EE[f(\x_t, \y^\star(\x_t))] - \EE[f(\x_t, \y_t)] \leq \frac{D^2}{\eta_\y N_t} + \eta_\y\sigma^2$ which implies that the number of iterations required is $O(\max\{1, 4\ell D^2\zeta^{-1}, 4\sigma^2 D^2 \zeta^{-2}\})$ which is also the number of stochastic gradient evaluations since $M=1$. 
\end{proof}
\paragraph{Proof of Theorem~\ref{Theorem:nc-SGDmax-complexity-bound}:} It is easy to find that the second descent inequality in Lemma~\ref{Lemma:nc-key-descent} is applicable to SGDmax: 
\begin{equation}\label{nc-SGDmax-first}
\EE[\Phi_{1/2\ell}(\x_t)] \leq \EE[\Phi_{1/2\ell}(\x_{t-1})] + 2\eta_\x\ell\Delta_{t-1} - \tfrac{\eta_\x}{4}\EE[\|\grad \Phi_{1/2\ell}(\x_{t-1})\|^2] + \eta_\x^2\ell(L^2 + \sigma^2).
\end{equation}
Summing up~\eqref{nc-SGDmax-first} over $T = 1, 2, \ldots, T+1$ together with $\Delta_{t-1} \leq \zeta$ and rearranging the terms yields
\begin{equation*}
\tfrac{1}{T+1} \left(\sum_{t=0}^T \EE[\|\grad \Phi_{1/2\ell}(\x_t)\|^2]\right) \leq \tfrac{4(\Phi_{1/2\ell}(\x_0) - \EE[\Phi_{1/2\ell}(\x_{T+1})])}{\eta_\x(T+1)} + 8\ell\zeta + 4\eta_\x\ell (L^2 + \sigma^2). 
\end{equation*}
By the definition of $\eta_\x$ and $\widehat{\Delta}_\Phi$, we have
\begin{equation*}
\frac{1}{T+1} \left(\sum_{t=0}^T \EE[\|\grad \Phi_{1/2\ell}(\x_t)\|^2]\right) \leq \tfrac{48\ell(L^2 + \sigma^2)\widehat{\Delta}_\Phi}{\epsilon^2(T+1)} + 8\ell\zeta + \tfrac{\epsilon^2}{3}. 
\end{equation*}
This implies that the number of iterations required by Algorithm \ref{Algorithm:SGDmax} to return an $\epsilon$-stationary point is bounded by 
\begin{equation*}
O\left(\frac{\ell(L^2 + \sigma^2)\widehat{\Delta}_\Phi}{\epsilon^4}\right). 
\end{equation*}
Combining Lemma~\ref{Lemma:nc-SGDmax-neighborhood} gives the total gradient complexity of Algorithm \ref{Algorithm:GDmax}:
\begin{equation*}
O\left(\frac{\ell(L^2 + \sigma^2)\widehat{\Delta}_\Phi}{\epsilon^4}\max\left\{1, \ \frac{\ell^2D^2}{\epsilon^2}, \ \frac{\ell^2 D^2 \sigma^2}{\epsilon^4}\right\}\right). 
\end{equation*}
This completes the proof.

\end{document}